\documentclass[12pt]{article}
\usepackage[utf8]{inputenc}
\usepackage{cite}
\usepackage{amsfonts}
\usepackage{amsmath}
\usepackage{amsthm}
\usepackage{enumitem}
\usepackage{amssymb}
\usepackage{caption}
\usepackage{graphicx}
\usepackage{mathrsfs}
\usepackage{multirow}
\usepackage{xcolor}
\usepackage{setspace}
\usepackage{booktabs}
\usepackage{latexsym, mathdots, array, extarrows}
\usepackage{algpseudocode,float}
\usepackage{lipsum}
\usepackage{pifont}
\usepackage{diagbox}

\usepackage{amsmath,amsfonts,amsthm,amssymb}
\usepackage{bm,dsfont}
\usepackage{booktabs}
\usepackage{graphicx}
\usepackage{array}
\usepackage{xcolor}
\usepackage{sidecap}
\usepackage{listings}
\usepackage{times}
\usepackage{url}
\usepackage{hyperref}
\usepackage{subfigure}
\usepackage{makecell} 

\usepackage{algorithm}      
\usepackage{algpseudocode}  

\usepackage{cleveref}

\usepackage{geometry}
\newtheorem{definition}{Definition}
\newtheorem{theorem}{Theorem}

\newtheorem{lemma}{Lemma}
\newtheorem{assumption}{Assumption}

\newtheorem*{remark}{Remark}

\newcommand{\bS}{\bm{S}}
\newcommand{\bQ}{\bm{Q}}
\newcommand{\bH}{\bm{H}}
\newcommand{\bs}{\bm{s}}
\newcommand{\bx}{\bm{x}}
\newcommand{\bU}{\bm{U}}

\newcommand{\bA}{\bm{A}}
\newcommand{\bB}{\bm{B}}
\newcommand{\bu}{\bm{u}}
\newcommand{\by}{\bm{y}}
\newcommand{\bY}{\bm{Y}}

\newcommand{\bX}{\bm{X}}
\newcommand{\bW}{\bm{W}}
\newcommand{\bV}{\bm{V}}

\newcommand{\blam}{\bm{\lambda}}
\newcommand{\bSig}{\bm{\Sigma}}

\newcommand{\mL}{\mathcal{L}}
\newcommand{\bI}{\bm{I}}
\newcommand{\bP}{\bm{P}}
\newcommand{\bZ}{\bm{Z}}

\newcommand{\bC}{\bm{C}}
\newcommand{\bE}{\bm{E}}
\newcommand{\mW}{\mathcal{W}}
\newcommand{\bJ}{\bm{J}}
\newcommand{\bK}{\bm{K}}

\allowdisplaybreaks[4]

\title{ODELoRA: Training Low-Rank Adaptation by Solving Ordinary Differential Equations
}

\author{Yihang Gao\thanks{Department of Mathematics, National University of Singapore. Email: gaoyh@nus.edu.sg}\and Vincent Y. F.  Tan\thanks{Department of Mathematics and Department of Electrical and Computer Engineering, National University of Singapore. Email: vtan@nus.edu.sg}}

\date{}

\begin{document}

\maketitle

\begin{abstract}

Low-rank adaptation (LoRA) has emerged as a widely adopted parameter-efficient fine-tuning method in deep transfer learning, due to its reduced number of trainable parameters and lower memory requirements enabled by Burer–Monteiro factorization on adaptation matrices.
However, classical LoRA training methods treat the low-rank factor matrices individually and optimize them using standard gradient-based algorithms. Such decoupled optimization schemes are theoretically and empirically suboptimal, as they fail to fully exploit the intrinsic structure of the LoRA parameterization.
In this work, we propose a novel continuous-time optimization dynamic for LoRA factor matrices in the form of an ordinary differential equation (ODE) that emulates the gradient flow of full fine-tuning on the balanced manifold. 
We term this approach ODELoRA.
To faithfully track the trajectories of ODELoRA, we adopt well-established and theoretically grounded time-discretization schemes, including Euler and Runge--Kutta methods.
Our framework provides a unified ODE-based perspective for understanding and designing LoRA training algorithms. We establish linear convergence of the proposed method under strongly convex objectives for certain discretization schemes under mild conditions, and further extend our analysis to the matrix sensing setting. Moreover, we show that ODELoRA achieves stable feature learning, a property that is crucial for training deep neural networks at different scales of problem dimensionality.
Empirical results on matrix sensing tasks confirm the derived linear convergence behavior, and experiments on training physics-informed neural networks further demonstrate the superiority of ODELoRA over existing baselines, especially in the training stability.

\end{abstract}

\section{Introduction}
Among various parameter-efficient fine-tuning methods, low-rank adaptation (LoRA) has gained particular attention due to its simplicity, effectiveness, and broad applicability across modern neural architectures, including large language models~\cite{hu2022lora,dettmers2023qlora,zhang2023adaptive,hayou2024lora+}, multimodal vision-language models~\cite{hanparameter,zanella2024low}, graph neural networks~\cite{li2024adaptergnn,yang2025graphlora,wang2025elora}, and physics-informed neural networks~\cite{wang2025transfer,berman2024colora,cho2023hypernetwork}.
Although full fine-tuning remains a powerful approach, its substantial computational and memory overheads have motivated extensive research into parameter-efficient fine-tuning techniques.
By restricting task-specific updates to low-rank subspaces of weight matrices, LoRA dramatically reduces the number of trainable parameters while maintaining performance competitive with full fine-tuning. This low-rank parameterization significantly reduces memory consumption and computational cost, making LoRA particularly attractive for fine-tuning deep neural networks.

The effectiveness of LoRA is commonly attributed to the observation that task adaptation often resides in a low intrinsic-dimensional subspace of the full parameter space. By parameterizing weight updates as low-rank matrix products via the Burer--Monteiro factorization, LoRA significantly reduces the number of trainable parameters while retaining sufficient expressive capacity for adaptation.
From a theoretical perspective, Jang et al.~\cite{jang2024lora} show that under an overparameterization regime, in which  neural tangent kernel theory implies a linearized model for an extremely wide network, LoRA optimization exhibits no spurious local minima. Thus, standard gradient-based methods can result in convergence to global optima. 
Kim et al.~\cite{kim2025lora} further extend this analysis to more general loss functions satisfying a restricted strong convexity condition, which is considerably milder than the overparameterization assumption.
These works mainly characterize the landscape of LoRA objectives, focusing on the existence and properties of local and global minima, rather than explicitly analyzing the optimization trajectories or convergence behavior. In contrast, Soltanolkotabi et al.~\cite{soltanolkotabi2025implicit} study LoRA fine-tuning in the matrix sensing setting and identify a three-phase training dynamics: an initial subspace alignment phase, during which the column spaces of the learned factors align with the ground truth; a subsequent phase that escapes saddle points; and a final local refinement stage exhibiting fast convergence once the iterates are sufficiently close to the global solution.
Beyond these results, a growing body of work has investigated LoRA from various complementary perspectives, including optimization dynamics~\cite{xu2025understanding,stoger2021small,xiong2024how,liu2025on}, expressiveness and representational capacity~\cite{zeng2024the}, and generalization~\cite{zhu2024asymmetry,kratsios2025sharp}.

However, despite its widespread adoption in practice and growing theoretical understanding, LoRA fine-tuning has almost been exclusively performed using standard optimizers, such as (stochastic) gradient descent or Adam~\cite{kingma2015adam}; these optimizers are applied directly to the factor matrices.
This practice treats the LoRA factors as {\em individual} parameter matrices and ignores the intrinsic multiplicative structure introduced by the low-rank parameterization.
As a consequence, the resulting training dynamics can deviate substantially from those of full fine-tuning, which is widely regarded as the optimal fine-tuning procedure when computational resources are available. This training mismatch often leads to suboptimal empirical performance as noted in~\cite{shuttleworth2025lora,li2025flatlora}.
Moreover, Hayou et al.~\cite{hayou2024lora+} showed that stable feature learning in LoRA requires imbalanced step sizes between the two factor matrices. While effective in theory, this requirement makes LoRA fine-tuning less convenient in practice and highly sensitive to step-size tuning, increasing the risk of instability.

These observations have motivated recent efforts to design optimization methods that explicitly account for the structure of LoRA parameterization and address both stability and mismatch with full fine-tuning. Zhang and Pilanci~\cite{zhang2024riemannian} proposed a Riemannian optimization method that introduces geometry-aware preconditioning on the LoRA factors, leading to improved theoretical stability and empirical performance. 
Wang et al.~\cite{wang2025lorapro} further developed LoRA-Pro, which determines update directions by explicitly mimicking the gradient descent dynamics of full fine-tuning. 
Building on this line of work, Yu et al.~\cite{yu2025altlora} proposed an alternating optimization scheme that avoids solving the Sylvester equation and further enhances stability and computational efficiency.
More LoRA-specific algorithms can be found in~\cite{almansoori2025faster,wang2024lora}.

In this paper, we develop a new class of optimization algorithms for LoRA fine-tuning that closely tracks full fine-tuning dynamics and achieve enhanced stability using balanced manifold constraints.
We first introduce a {\em continuous-time optimization flow} for LoRA factors, termed {\em   ODELoRA}, which is a solution of equality constrained quadratic optimization problem, where LoRA training explicitly mimics the gradient flow of full fine-tuning on a balanced manifold.
To obtain practical algorithms, we discretize ODELoRA using well-established numerical solvers for ordinary differential equations, including Euler and Runge--Kutta methods, with explicit theoretical guarantees on discretization error. In particular, higher-order Runge--Kutta schemes allow the discrete iterates to more accurately track the continuous ODELoRA trajectory, making them especially suitable for careful and stable fine-tuning.
Our main contributions are summarized as follows:
\begin{enumerate}[label=(\arabic*)]
    \item We propose a continuous-time optimization flow for LoRA fine-tuning, termed \emph{ODELoRA}, which faithfully mimics the gradient flow of full fine-tuning and satisfies balanced manifold constraints to mitigate gauge effects.

    \item We design practical LoRA optimization algorithms by discretizing ODELoRA using classical numerical solvers, including Euler and Runge--Kutta methods.

    \item We establish linear convergence guarantees for discretized ODELoRA under strongly convex objectives, and further extend the analysis to the matrix sensing setting. Moreover, we show that ODELoRA  achieves \emph{stable feature learning}, a critical property for scalable and reliable LoRA optimization.

    \item We empirically validate the theoretical results on matrix sensing problems, demonstrating that ODELoRA with Runge--Kutta 4 discretization exhibits superior stability and most closely tracks full fine-tuning.
    We further apply the proposed method to fine-tuning physics-informed neural networks, where ODELoRA consistently outperforms existing baseline optimization methods.
\end{enumerate}

The remainder of the paper is organized as follows. In \Cref{section2}, we review preliminaries on LoRA parameterization and numerical solvers for ordinary differential equations. The proposed ODELoRA algorithms are introduced in \Cref{section3}. Theoretical analyses are presented in \Cref{section4}. Experimental results are reported in \Cref{section5}. Finally, conclusions and directions for future work are discussed in \Cref{section6}.

\section{Preliminaries}
\label{section2}
In this section, we first introduce the notation used throughout the paper. We then present the mathematical formulation of LoRA and review the associated optimization algorithms. Finally, we briefly discuss well-established numerical solvers for ordinary differential equations (ODEs), which motivate the design of the proposed method.

\subsection{Notation}
Throughout the paper, bold lowercase letters (e.g., $\bx$) denote vectors and bold uppercase letters (e.g., $\bA$) represent matrices.
Scalars are represented by regular (non-bold) lowercase letters (e.g., $a$). For continuous-time dynamics, we use $\bW(t)$ (and $\bA(t)$, $\bB(t)$) to denote time-dependent variables, whereas their discretized counterparts are denoted by $\bW_t$ (and $\bA_t$, $\bB_t$). 
For notational simplicity, we abbreviate $\mL(\bW_t)$ as $\mL_t$ and the optimal objective value $\mL(\bW^{\star})$ at the optimal point $\bW^{\star}$ as $\mL^{\star}$.
We use $\sigma_{\max}(\bA)$ and $\sigma_{\min}(\bA)$ to denote the largest and smallest singular value of the given matrix $\bA$, respectively.
We use $\left\|\cdot \right\|$ to denote the Euclidean norm of vectors and matrix spectral norm.

\subsection{Low-Rank Adaptation}
We consider the following optimization problem for model training:
\begin{equation}
\label{eq_strongly_convex}
    \min_{\bW} \mL(\bW),
\end{equation}
where $\mL$ denotes the objective function. 
In low-rank adaptation (LoRA)~\cite{hu2022lora}, the model parameters $\bW_{\text{pt}}$ are fine-tuned using the following Burer--Monteiro factorization:
\begin{equation*}
    \bW_{\text{ft}} = \bW_{\text{pt}} + \bB \bA,
\end{equation*}
where $\bW_{\text{ft}}, \bW_{\text{pt}} \in \mathbb{R}^{m \times n}$ denote the fine-tuned and pre-trained model parameters, respectively, $\bA \in \mathbb{R}^{r \times n}$ and $\bB \in \mathbb{R}^{m \times r}$ are factor matrices of LoRA, and $r \ll \min\{m, n\}$. Here, the pre-trained parameters $\bW_{\text{pt}}$ are kept fixed, and only the low-rank factors $\bA$ and $\bB$ are trainable. 
Accordingly, under the LoRA parameterization, we instead solve
\begin{equation}
\label{eq_strongly_convex_lora}
    \min_{\bA, \bB} \mL(\bW_{\text{pt}} + \bB \bA).
\end{equation}
This formulation significantly reduces the number of trainable parameters and, consequently, the computational and memory costs of fine-tuning. As a result, LoRA has emerged as one of the most effective and widely adopted parameter-efficient fine-tuning methods.
For notational simplicity, we represent the LoRA parameterization in terms of a single weight matrix $\bW$, but the same formulation and derivations extend naturally and directly to all parameter matrices of the given model.

Classical optimization methods for training LoRA treat the two low-rank factors $\bA$ and $\bB$ as {\em separate} parameter matrices and apply gradient-based updates to each factor {\em individually}.
Specifically, let $(\bA_t, \bB_t)$ denote the parameter factors at the $t$-th iteration. The standard update scheme is given by
\begin{equation}
\label{eq1}
    \begin{split}
        \bA_{t+1} = \bA_t - \eta  \nabla_{\bA}\mL(\bW_{\text{pt}} + \bB_t \bA_t),\\
        \bB_{t+1} = \bB_t - \eta  \nabla_{\bB}\mL(\bW_{\text{pt}} + \bB_t \bA_t),\\
    \end{split}
\end{equation}
where $\eta>0$ denotes the step size. 
However, this update scheme is generally suboptimal, as it ignores the intrinsic multiplicative structure of the low-rank parameterization $\bB \bA$. 
As pointed out in \cite{shuttleworth2025lora,li2025flatlora}, LoRA trained with \eqref{eq1} can deviate substantially from full fine-tuning, leading to unsatisfactory performance due to the mismatch between the factorized optimization dynamics and those of the full-weight space. 
Moreover, prior work~\cite{xiong2024how} has shown that vanilla gradient descent on the factorized formulation can induce imbalanced magnitudes between $\bA$ and $\bB$, which in turn leads to instability in training.
Similarly, Hayou et al.~\cite{hayou2024lora+} characterize efficient and stable fine-tuning for LoRA and demonstrate that the classical update rule in \eqref{eq1} is inherently unstable, requiring step sizes that scale polynomially with the matrix dimensions, an undesirable condition for practical applications.

To address the step-size imbalance and achieve efficient and stable fine-tuning using a unified step sizes for updating $\bA$ and $\bB$, Riemannian LoRA~\cite{zhang2024riemannian} introduces a preconditioning strategy for the gradients. Specifically, the update directions are defined as
\begin{equation*}
    \Delta\bA_{t} = -\left(\bB_t^{\top} \bB_t\right)^{-1} \nabla_{\bA}\mL(\bW_{\text{pt}} + \bB_t \bA_t), \quad \Delta\bB_{t} = - \nabla_{\bB}\mL(\bW_{\text{pt}} + \bB_t \bA_t)\left(\bA_t \bA_t^{\top}\right)^{-1}.
\end{equation*}
The parameters are then updated according to
\begin{equation}
\label{eq3}
        \bA_{t+1} = \bA_t - \eta \Delta\bA_t, \quad \bB_{t+1} = \bB_t - \eta \Delta\bB_t.
\end{equation}
Compared to the classical LoRA updates, Riemannian LoRA effectively preconditions the gradients to account for the geometry induced by the low-rank factorization. This preconditioning balances the gradient magnitudes of $\bA$ and $\bB$, leading to improved numerical convergence speed. Nevertheless, despite these advantages, Riemannian LoRA still deviates from the full fine-tuning dynamics in the original parameter space, which may lead to performance degradation in practice.

Recently, Wang et al.~\cite{wang2025lorapro} proposed LoRA-Pro, an optimization algorithm for LoRA that explicitly aims to mimic the behavior of full fine-tuning with gradient descent. Instead of directly using negative gradients as update directions, LoRA-Pro determines the update directions by solving the following optimization problem:
\begin{equation}
\label{eq2}
    \min_{\Delta\bA_t, \Delta\bB_t} \left\| \bB_{t} \Delta\bA_t + \Delta\bB_t \bA_t + \nabla_{\bW}\mL\left(\bW_{\text{pt}} + \bB_t \bA_t\right) \right\|_{\mathrm{F}}^2.
\end{equation}
The parameters $(\bA,\bB)$ are then updated according to \eqref{eq3}, where $(\Delta\bA_t,\Delta\bB_t)$ are obtaind by solving~\eqref{eq2}.
The motivation behind LoRA-Pro arises from examining how the effective weight matrix $\bW_{\text{ft}}$ under LoRA updates matches full fine-tuning. 
Specifically, after one update step, the change in the full matrix is given by
\begin{equation}
\label{eq4}
    \begin{split}
        \Delta\bW & = \left(\bB_t + \eta\Delta\bB_t\right) \left(\bA_t + \eta\Delta\bA_t\right) - \bB_t \bA_t \\
        & = \eta\left(\bB_{t} \Delta\bA_t +  \Delta\bB_t \bA_t\right) + \eta^2 \Delta\bB_{t} \Delta\bA_t,
    \end{split}
\end{equation}
which generally differs from the update induced by full fine-tuning with
    $\Delta\bW^{\text{full}} = -\eta \nabla_{\bW}\mL\left(\bW_{\text{ft}}\right)$.
LoRA-Pro seeks to bridge this gap by matching the first-order term in $\eta$ from \eqref{eq4} with the full fine-tuning update, while neglecting the second-order term. In this way, LoRA-Pro approximates full fine-tuning with gradient descent under the low-rank parameterization.

\subsection{Numerical Solvers for ODEs}
We consider the ODE
\begin{equation}
\label{eq6}
    \frac{d \bu}{dt} = F(\bu), \quad \bu(0) = \bu_0,
\end{equation}
where $F$ is a vector-valued functional and $\bu_0$ denotes the initial condition. Numerical methods for solving \eqref{eq6} are well studied and rely on discretizing the continuous-time dynamics using various time-discretization schemes~\cite{leveque2007finite,rosser1967runge,Hildebrand1987}.

The simplest approach is Euler's method, which yields the first-order discretization
\begin{equation}
    \label{eq7}
    \bu_{t+1} = \bu_{t} + h F(\bu_t),
\end{equation}
where $h>0$ is the step size, and $\bu_t$ approximates the solution $\bu(th)$. 
Under standard smoothness assumptions on $F$, Euler's method achieves first-order accuracy, meaning that the global discretization error satisfies
\begin{equation*}
    \left\|\bu_{t} - \bu(th) \right\| = \mathcal{O}(h),
\end{equation*}
where constants depending on the regularity and smoothness of $F$ and $\bu$ are omitted.

Higher-order accuracy can be obtained using Runge–Kutta (RK) methods, which are among the most widely used numerical solvers for ODEs~\cite{leveque2007finite,rosser1967runge,Hildebrand1987}. The second-order Runge–Kutta method (RK2), also known as Heun's method, is given by
\begin{equation}
    \label{eq8}
    \begin{split}
        & \bu_{t}^{(1)} = \bu_{t} + h F(\bu_t),\\ 
        & \bu_{t+1} = \bu_{t} + \frac{h}{2} \left(F(\bu_t) + F(\bu_{t}^{(1)})\right).
    \end{split}
\end{equation}
The classical fourth-order Runge–Kutta method (RK4) takes the form
\begin{equation}
    \label{eq30}
    \begin{split}
        & \bu_{t}^{(1)} = \bu_{t} + \frac{h}{2} F(\bu_t),\\ 
        & \bu_{t}^{(2)} = \bu_{t} + \frac{h}{2} F(\bu_{t}^{(1)}),\\
        & \bu_{t}^{(3)} = \bu_{t} + h F(\bu_{t}^{(2)}),\\
        & \bu_{t+1} = \bu_{t} + \frac{h}{6} \left(F(\bu_t) + 2F(\bu_{t}^{(1)}) + 2 F(\bu_{t}^{(2)}) + F(\bu_{t}^{(3)})\right).
    \end{split}
\end{equation}
Compared to Euler's method, RK2 and RK4 achieve higher-order accuracy, with global discretization errors bounded by $\left\|\bu_{t} - \bu(th) \right\| = \mathcal{O}(h^2)$ and $\left\|\bu_{t} - \bu(th) \right\| = \mathcal{O}(h^4)$, respectively. 
Although these methods introduce additional computational costs per iteration, their improved accuracy often allows for larger step sizes while maintaining stability. Consequently, RK2 and RK4 provide a favorable trade-off between computational efficiency and numerical accuracy, and are widely adopted in practical applications. Motivated by these considerations, we focus on RK2 and RK4 as representative higher-order discretization schemes in this work.

\section{The Proposed Method}
\label{section3}
In this section, we first present the motivation and introduce the proposed training dynamics for LoRA, termed ODELoRA, which are formulated as an ordinary differential equation. Based on this continuous-time formulation, we then develop optimization algorithms for LoRA by discretizing the resulting ODE using several numerical schemes, including Euler's method, RK2, and RK4. Finally, we discuss the relationships between the proposed method and existing optimization algorithms.

\subsection{Motivation}
It is well known that full fine-tuning generally achieves state-of-the-art performance, which cannot, in general, be fully matched by LoRA-based fine-tuning~\cite{shuttleworth2025lora}. This observation motivates the principle that effective training dynamics for LoRA should closely approximate those of full fine-tuning. Accordingly, inspired by LoRA-Pro and its gradient-matching perspective, we propose to design LoRA training dynamics that directly emulate the gradient flow of the full parameter matrix.
Specifically, instead of optimizing the low-rank factors separately, we consider the following continuous-time optimization problem:
\begin{equation}
\label{eq9}
    \min_{\bJ(t), \bK(t)}\left\| \bK(t) \bA(t) + \bB(t)\bJ(t) - \left(- \nabla_{\bW} \mL(t)\right) \right\|_{\mathrm{F}}^{2},
\end{equation}
where $(\bA(t),\bB(t))$ denote the continuous-time counterparts of the low-rank factors $(\bA,\bB)$ in \eqref{eq_strongly_convex_lora}, and we define $\bW(t) := \bW_{\text{pt}}+ \bB(t)\bA(t)$ and $\mL(t):= \mL(\bW(t))$.
Then, the continuous-time dynamics is determined by
\begin{equation}
\label{eq35}
    \frac{d \bA(t)}{dt} = \bJ^{\star}(t),\quad \frac{d \bB(t)}{dt} = \bK^{\star}(t),
\end{equation}
where $\left( \bJ^{\star}(t),  \bK^{\star}(t)\right)$ is a minimizer of \eqref{eq9}.
Note that $\frac{d \bW(t)}{d t} = \frac{d \bB(t)}{dt} \bA(t) + \bB(t)\frac{d \bA(t)}{dt}$, which implies that \eqref{eq9} enforces the LoRA-induced dynamics of the full matrix $\bW(t)$ to best match the negative gradient flow $-\nabla_{\bW}\mL(t)$ in terms of the Frobenius norm.  
In this way, the proposed formulation shifts the focus from the steepest descent directions of $\bA$ and $\bB$ individually to that of the full parameter matrix $\bW$, which explicitly respect the low-rank multiplicative structure.
The objective in \eqref{eq9} is convex (since the expression inside the Frobenius norm is linear in the variables $\left(\bJ(t),\bK(t)\right)$) but not strongly convex with respect to $\left(\bJ(t), \bK(t)\right)$, and thus admits infinitely many solutions. By applying first-order optimality conditions, the solution can be expressed in closed form as 
\begin{equation*}
    \begin{split}
        & \bJ^{\star}(t)  =  -\left(\bB(t)^{\top}\bB(t)\right)^{-1} \bB(t)^{\top} \nabla_{\bW} \mL(t) + \bX(t) \bA(t),\\
        & \bK^{\star}(t) = - \left[\bI - \bB(t) \left(\bB(t)^{\top}\bB(t)\right)^{-1}\bB(t)^{\top}\right] \nabla_{\bW} \mL(t) \bA(t)^{\top} \left(\bA(t) \bA(t)^{\top}\right)^{-1} - \bB(t) \bX(t),
    \end{split}
\end{equation*}
where $\bX(t)$ is an arbitrary matrix characterizing the non-uniqueness of the factorization in \eqref{eq9}.

Although different choices of $\bX(t)$ produce the same dynamics of the full parameter matrix $\frac{d \bW(t)}{d t} $, the underlying factor dynamics for $(\bA,\bB)$ can differ substantially due to the multiplicative structure of LoRA. 
In particular, an undesired $\bX(t)$ corresponds to cancelling updates between $\bA(t)$ and $\bB(t)$, resulting in redundant motion in factor space that does not contribute to progress in the full parameter matrix. Such pathological trajectories not only waste optimization effort but may also introduce numerical instability. 
This phenomenon is commonly referred to as gauge freedom in Burer–Monteiro factorizations, as the LoRA parameterization admits the invariance
\begin{equation*}
    \bB\bA = \left(\bB\bQ\right)\left(\bQ^{-1}\bA\right),
\end{equation*}
for any invertible matrix $\bQ \in \text{GL}(r)$. When $\bQ$ is nearly singular, this reparameterization can cause severe imbalance between the two factor matrices, amplifying numerical instability despite leaving their product unchanged.
Consequently, specifying the optimization flow of  $(\bA(t),\bB(t))$ solely through \eqref{eq9}, even when it matches the full fine-tuning dynamics, is insufficient. A well-behaved factorization must additionally account for the gauge freedom and explicitly avoid trajectories approaching singular reparameterizations. Although such transformations are equivalent at the level of the product $\bB\bA$, they can yield highly unstable and inefficient dynamics in the factor space.

Therefore, to address the issue of nearly singular gauges, we impose additional manifold constraints that enforce balanced factorizations. Specifically, we consider the manifold
\begin{equation}
\label{eq34}
    \mathcal{M}= \left\{(\bA,\bB) \in \mathbb{R}^{r \times n} \times \mathbb{R}^{m \times r}: \bA\bA^{\top} = \bB^{\top}\bB\right\},
\end{equation}
and assume that the trajectory $(\bA(t),\bB(t))$ remains on $\mathcal{M}$ for all $t\geq0$.
This constraint selects a balanced representative within the gauge-equivalence class and is directly motivated by balancing regularizers commonly used in classical low-rank matrix factorization; see, e.g., \cite{ge2017no,park2017non}.

\begin{lemma}
    The balanced manifold $\mathcal{M}$ preserves the expressiveness of the LoRA parameterization, i.e., for any matrix $\bW \in \mathbb{R}^{m \times n}$ with rank $r$, there exists $(\bA,\bB)\in\mathcal{M}$ such that $\bW = \bB \bA$.
\end{lemma}
\begin{proof}
    Let $\bW \in \mathbb{R}^{m \times n}$ be any matrix of rank $r$. By the singular value decomposition, $\bW$ admits the representation $\bW = \bU_{r}\bSig_{r}\bV_r^{\top}$, where $\bSig_{r} = \mathrm{diag}(\sigma_1, \cdots,\sigma_r) \in \mathbb{R}^{r \times r}$ contains the top-$r$ singular values, and the columns of $\bU_{r} \in \mathbb{R}^{m \times r}$ and $\bV_r \in \mathbb{R}^{n \times r}$ are the corresponding singular vectors. Define $\bSig_{r}^{1/2}:= \mathrm{diag}(\sqrt{\sigma_1}, \cdots,\sqrt{\sigma_r})$, and set $\bA = \bSig_{r}^{1/2}\bV_r^{\top}$ and $\bB=\bU_r \bSig_{r}^{1/2}$. Then, $\bW = \bB\bA$, and $\bA\bA^{\top} = \bB^{\top}\bB = \bSig_{r}$, which implies $(\bA,\bB) \in \mathcal{M}$. Hence, any rank-$r$ matrix admits a balanced LoRA factorization, and the expressiveness is preserved.
\end{proof}
Differentiating the manifold constraint \eqref{eq34} with respect to $t$, we have that
\begin{equation*}
    \frac{d \bA(t)}{dt} \bA(t)^{\top} + \bA(t) \frac{d \bA(t)^{\top} }{dt} - \frac{d \bB(t)^{\top}}{dt} \bB(t)^{\top} + \bB(t)^{\top} \frac{d \bB(t) }{dt} = \bm{0}.
\end{equation*}
Rather than determining the optimization flow solely through \eqref{eq9}, we therefore impose this additional balanced manifold constraint and define the following constrained optimization problem:
\begin{equation}
\label{eq31}
\begin{split}
    & \min_{\bJ(t), \bK(t)}\left\| \bK(t) \bA(t) + \bB(t)\bJ(t) - \left(- \nabla_{\bW} \mL(t)\right) \right\|_{\mathrm{F}}^{2},\\
    & \text{s.t.}\quad \bJ(t) \bA(t)^{\top} + \bA(t) \bJ(t)^{\top} - \bK(t) \bB(t)^{\top} + \bB(t)^{\top} \bK(t)^{\top} = \bm{0}.
\end{split}
\end{equation}
Then, the continuous-time dynamics \eqref{eq35} for LoRA factors is obtained by $\left( \bJ^{\star}(t),  \bK^{\star}(t)\right)$, which is one of the solutions of \eqref{eq31}.

Compared with the continuous-time dynamics of LoRA-Pro in \eqref{eq9}, the proposed dynamics in \eqref{eq31} enforce stringent constraints on the balanced manifold $\mathcal{M}$, therefore, eliminating the undesirable gauge effect. By prioritizing factor balance, the resulting flow then mimics full fine-tuning dynamics within the class of balanced factorizations.
The following theorem characterizes the explicit solution of the constrained optimization problem \eqref{eq31}. 
\begin{theorem}
\label{theorem5}
Assume that $\bA(t)\bA(t)^{\top}$ and $\bB(t)^{\top}\bB(t)$ are positive definite.
Then the constrained optimization problem \eqref{eq31}, which is a convex problem with linear equality constraints, admits a solution of the form
    \begin{equation}
    \label{eq32}
    \begin{split}
        & \bJ^{\star}(t)  =  -\left(\bB(t)^{\top}\bB(t)\right)^{-1} \bB(t)^{\top} \nabla_{\bW} \mL(t) + \bX(t) \bA(t),\\
        & \bK^{\star}(t)  = - \left[\bI - \bB(t) \left(\bB(t)^{\top}\bB(t)\right)^{-1}\bB(t)^{\top}\right] \nabla_{\bW} \mL(t) \bA(t)^{\top} \left(\bA(t) \bA(t)^{\top}\right)^{-1} - \bB(t) \bX(t),
    \end{split}
\end{equation}
where $\bX(t)$ is the unique solution to the Sylvester equation
\begin{equation}
\label{eq33}
    \bH(t)\bX(t) + \bX(t)\bH(t) = \left(\bB(t)^{\top}\bB(t)\right)^{-1}\bB(t)^{\top} \nabla_{\bW} \mL(t)\bA(t)^{\top} + \bA(t)\nabla_{\bW} \mL(t)^{\top} \bB(t)\left(\bB(t)^{\top}\bB(t)\right)^{-1},
\end{equation}
with $\bH(t):=\bA(t)\bA(t)^{\top} + \bB(t)^{\top}\bB(t)$. 
\end{theorem}

The proof is deferred to \Cref{appendix1}.
The objective in \eqref{eq31} suppresses these undesirable gauge-induced behaviors and leads to smooth and well-conditioned dynamics.
The auxiliary variable $\bX(t)$ captures motion along these invariant (or null) directions that leave $\bW$ unchanged.
By minimizing the update norm under manifold constraints, problem \eqref{eq31} effectively fixes a gauge, selecting a canonical flow among all equivalent factor trajectories on the balanced manifold.
As a consequence, the resulting dynamics avoid unnecessary drift along invariant directions and produce coupling-aware updates for $\bA$ and $\bB$. This leads to balanced factor dynamics with improved numerical stability, while faithfully tracking the desired full fine-tuning gradient flow.

\subsection{ODELoRA and its Numerical Solvers}

Rather than treating the LoRA training dynamics in a decoupled structure, we jointly consider the dynamics characterized by \eqref{eq31} with manifold constraints. 
The right-hand side of \eqref{eq32} depends on $\bA(t)$, $\bB(t)$, and the auxiliary variable $\bX(t)$.
Meanwhile, under standard regularity conditions, \eqref{eq33} together with the implicit function theorem implies that $\bX(t)$ can be expressed as a function of $\bA(t)$ and $\bB(t)$, provided that $\bB(t)^{\top}\bB(t)$ is positive definite~\cite{higham2008functions}. As a result, the right-hand side of \eqref{eq32} becomes a functional of $\bA(t)$ and $\bB(t)$ alone.
For notational simplicity, we model the resulting LoRA training dynamics as the following ordinary differential equation, referred to as ODELoRA:
\begin{equation}
\label{eq13}
    \frac{d}{dt} \left(\begin{array}{c}
        \bA(t) \\
        \bB(t)
    \end{array}\right) = F(\bA(t),\bB(t)),
\end{equation}
where 
\begin{equation}
\label{eq5}
    \begin{split}
        & F\left(\bA(t),\bB(t)\right) := \left(\begin{array}{c}
             F_{\bA}\left(\bA(t),\bB(t)\right)  \\
             F_{\bB}\left(\bA(t),\bB(t)\right) 
        \end{array}\right) \\
        & := \left(\begin{array}{c}
        -\left(\bB(t)^{\top}\bB(t)\right)^{-1} \bB(t)^{\top} \nabla_{\bW} \mL(t) + \bX(t) \bA(t) \\
        - \left[\bI - \bB(t) \left(\bB(t)^{\top}\bB(t)\right)^{-1}\bB(t)^{\top}\right] \nabla_{\bW} \mL(t) \bA(t)^{\top} \left(\bA(t) \bA(t)^{\top}\right)^{-1} - \bB(t) \bX(t)
    \end{array}\right) 
    \end{split}
\end{equation}
is defined explicitly through \eqref{eq32}, with $\bX(t)$ determined by the Sylvester equation \eqref{eq33}.

The ODELoRA formulation in \eqref{eq13} admits a natural geometric interpretation. Specifically, it can be viewed as a projected gradient flow of full fine-tuning onto the low-rank manifold under the balanced manifold constraints $\mathcal{M}$ (in \eqref{eq34}) induced by the LoRA parameterization. The operators $\left(\bB(t)^{\top}\bB(t)\right)^{-1}\bB(t)^{\top}$ and $\left(\bA(t)\bA(t)^{\top}\right)^{-1}\bA(t)$ correspond to least-squares optimal projections of the full-model gradient $\nabla_{\bW} \mL(t)$ onto the column space of $\bB(t)$ and the row spaces of $\bA(t)$, respectively.  
Consequently, ODELoRA produces the closest feasible update to the full gradient flow under the balance manifold constraint and LoRA's parameterization structure, and can be interpreted as a natural gradient flow with respect to the geometry induced by the low-rank and balanced factorization.

Under the assumption that $\bA(t)$ and $\bB(t)$ remain full rank $r$ at all time step $t$, the right-hand side of \eqref{eq32} is differentiable with respect to $(\bA(t),\bB(t),\bX(t))$. 
Moreover, if  $\bA(t)\bA(t)^{\top}$ and $\bB(t)^{\top}\bB(t)$ are positive definite, standard results on Sylvester equations imply that $\bX(t)$ is unique and is differentiable with $(\bA(t),\bB(t))$~\cite{higham2008functions}. 
Therefore, under these mild regularity conditions and assume that $\mL$ is differentiable in $\bW$, $F(\bA,\bB)$ in \eqref{eq13} is differentiable with respect to $(\bA,\bB)$.
This regularity ensures that ODELoRA defines a smooth and locally Lipschitz ODE on the parameter space $(\bA,\bB)$.

Formulating LoRA training as a continuous-time dynamical system therefore enables the direct application of well-established numerical ODE solvers with proven stability and convergence guarantees. In particular, standard explicit methods such as the forward Euler method and Runge--Kutta schemes of order two and four can be adopted to discretize the ODELoRA dynamics. 
Following \eqref{eq7}, applying Euler's method to ODELoRA (referred to as ODELoRA–Euler) yields the update
\begin{equation}
    \label{eq14}
    \left(\begin{array}{c}
         \bA_{t+1}\\
         \bB_{t+1} 
    \end{array}\right) = 
    \left(\begin{array}{c}
         \bA_{t}\\
         \bB_{t} 
    \end{array}\right) + h F(\bA_t,\bB_t).
\end{equation}
Applying the second-order Runge--Kutta method \eqref{eq8} to ODELoRA leads to the following update scheme, termed ODELoRA–RK2:
\begin{equation}
    \label{eq15}
    \begin{split}
        & \left(\begin{array}{c}
         \bA_{t}^{(1)}\\
         \bB_{t}^{(1)} 
    \end{array}\right) = 
    \left(\begin{array}{c}
         \bA_{t}\\
         \bB_{t} 
    \end{array}\right) + h F(\bA_t,\bB_t),\\ 
        & \left(\begin{array}{c}
         \bA_{t+1}\\
         \bB_{t+1} 
    \end{array}\right) = \left(\begin{array}{c}
         \bA_{t}\\
         \bB_{t} 
    \end{array}\right) + \frac{h}{2} \left(F(\bA_t,\bB_t) + F(\bA_{t}^{(1)},\bB_{t}^{(1)})\right).
    \end{split}
\end{equation}
Similarly, applying the fourth-order Runge--Kutta method \eqref{eq30} leads to ODELoRA–RK4, which updates the parameters as
\begin{equation}
    \label{eq16}
    \begin{split}
        & \left(\begin{array}{c}
         \bA_{t}^{(1)}\\
         \bB_{t}^{(1)} 
    \end{array}\right) = \left(\begin{array}{c}
         \bA_{t}\\
         \bB_{t} 
    \end{array}\right) + \frac{h}{2} F\left(\bA_t,\bB_t\right),\\ 
        & \left(\begin{array}{c}
         \bA_{t}^{(2)}\\
         \bB_{t}^{(2)} 
    \end{array}\right) = \left(\begin{array}{c}
         \bA_{t}\\
         \bB_{t} 
    \end{array}\right) + \frac{h}{2} F\left(\bA_{t}^{(1)},\bB_{t}^{(1)}\right),\\
        & \left(\begin{array}{c}
         \bA_{t}^{(3)}\\
         \bB_{t}^{(3)} 
    \end{array}\right) = \left(\begin{array}{c}
         \bA_{t}\\
         \bB_{t} 
    \end{array}\right) + h F\left(\bA_{t}^{(2)},\bB_{t}^{(2)}\right),\\
        & \left(\begin{array}{c}
         \bA_{t+1}\\
         \bB_{t+1} 
    \end{array}\right) = \left(\begin{array}{c}
         \bA_{t}\\
         \bB_{t} 
    \end{array}\right) + \frac{h}{6} \left(F\left(\bA_t,\bB_t\right) + 2F\left(\bA_{t}^{(1)},\bB_{t}^{(1)}\right) \right.\\
    & \hspace{10em} \left. + 2 F\left(\bA_{t}^{(2)},\bB_{t}^{(2)}\right) + F\left(\bA_{t}^{(3)},\bB_{t}^{(3)} \right)\right).
    \end{split}
\end{equation}


From a dynamical perspective, our ODE flow additionally considers the balanced manifold constraints, leading to balanced factor matrices and stable optimization, compared to the continuous-time flow for LoRA-Pro.
Although the dynamics of the effective weight matrix $\bW(t)$ coincide for ODELoRA and continuous LoRA-Pro, the additional manifold constraints imposed on $\frac{d \bA(t)}{dt}$ and $\frac{d \bB(t)}{dt}$ regulate the energy distribution between $\bA(t)$ and $\bB(t)$, mitigating gauge degeneracy and enhancing numerical stability for $\bA(t)$ and $\bB(t)$ individually.

From a numerical analysis standpoint, higher-order Runge–Kutta methods also exhibit improved stability compared to the forward Euler scheme, particularly when larger step sizes are used. By sampling $F$ at intermediate points, RK methods reduce local truncation errors and alleviate sensitivity to step-size selection, leading to smoother and more stable sequence of the factor parameters.

In practice, ODELoRA–Euler serves as a lightweight baseline that is well suited for large-scale training or warm-up phases. ODELoRA–RK2 offers a favorable balance between computational overhead and approximation accuracy, while ODELoRA–RK4, despite its higher cost, provides a high-fidelity discretization for scenarios where precise tracking of the continuous-time dynamics and refined fine-tuning performance are desired.

\subsection{Relations to Other Algorithms}

We first clarify the relationship between ODELoRA–Euler and existing LoRA optimization methods. In particular, from an algorithmic perspective, ODELoRA–Euler differs from the LoRA-Pro algorithm mainly in the Sylvester equation~\eqref{eq33}, which modifies gradient descent on $(\bA,\bB)$ by solving the auxiliary optimization problem in \eqref{eq2}.
The primary motivation of LoRA-Pro is to align the update direction of LoRA with that of full fine-tuning in the gradient descent setting. As a result, LoRA-Pro does not explicitly formulate LoRA training as a continuous-time dynamical system, even though its update rule shares some similarities with the forward Euler discretization of ODELoRA. This algorithm similarity follows from the well-known fact that gradient descent corresponds to the Euler discretization of gradient flow.
In contrast, our formulation in~\eqref{eq31} explicitly models LoRA optimization as a continuous-time flow on the parameter space and incorporates the balanced manifold constraint. This constraint enforces strict balance between the two low-rank factors, thereby mitigating the gauge freedom inherent in the factorized parameterization. As a result, \textsc{ODELoRA} yields more stable optimization dynamics and admits an interpretation in terms of constrained continuous-time optimization.


Second, ODELoRA–RK2 can be interpreted through the lens of several well-established optimization techniques, including extragradient methods, consensus optimization, and stochastic gradient averaging (SGA), when applied to LoRA with the modified gradient structure induced by \eqref{eq2}. These methods were originally developed to stabilize optimization in settings with coupled variables or non-conservative gradient fields. Since the LoRA parameterization introduces an intrinsic bilinear coupling between $\bA$ and $\bB$, the resulting dynamics deviate from standard gradient descent behavior. The midpoint evaluation in ODELoRA–RK2 and ODELoRA-RK4 effectively perform a look-ahead step along the LoRA flow, using intermediate gradient information to correct the update direction. This mechanism closely aligns with the core principle of extragradient-type methods, which aim to improve stability and convergence by anticipating future dynamics.

Third, ODELoRA–RK4 can be interpreted as a higher-order numerical approximation of the underlying LoRA flow, rather than as a heuristic modification of the optimization procedure. Unlike momentum-based or adaptive methods, which alter the effective objective or rescale gradients, RK4 improves optimization performance by reducing the discretization error in approximating the continuous-time dynamics. Since most existing optimization algorithms correspond to first- or second-order discretizations, we consider the higher-order RK4 scheme to better capture the more intricate dynamics of ODELoRA. As a result, ODELoRA–RK4 more faithfully tracks the ideal LoRA gradient flow induced by \eqref{eq14}, leading to smoother parameter trajectories and enhanced numerical stability, particularly when larger step sizes are used and the dynamics are complex.

More broadly, the ODE-based formulation provides a unifying perspective for understanding and designing LoRA training algorithms. From this viewpoint, existing methods such as LoRA-Pro can be interpreted as specific low-order discretizations of a continuous-time ODE, while higher-order numerical solvers naturally give rise to new algorithms with improved numerical properties. This unifying framework clarifies the relationships among seemingly different LoRA variants and offers a foundation for developing future methods by leveraging advances in numerical ODE solvers.

\section{Theoretical Analysis}
\label{section4}
The theoretical properties of ODELoRA and its numerical solvers, including Euler, RK2, and RK4, remain largely unexplored. Existing works primarily emphasize algorithmic design or empirical performance, and typically lack a rigorous theoretical foundation.

In this section, we establish global convergence guarantees for ODELoRA under strongly convex objectives. We further specialize our analysis to matrix sensing problems, which serve as a canonical and analytically tractable model for LoRA training. Building on these results, we analyze the convergence behavior of ODELoRA under discretization and show that both Euler and Runge–Kutta solvers inherit similar convergence properties from the underlying continuous-time dynamics. Finally, we demonstrate that the proposed ODELoRA methods exhibit favorable stability and efficiency properties.

\subsection{ODELoRA under Strongly Convex Objectives}
Although the strong convexity condition may not hold in practice, theoretical analysis of LoRA under this setting provides valuable insights into its optimization behavior and underlying mechanisms.
We consider the problem \eqref{eq_strongly_convex} with an objective function $\mL$ that is strongly convex in $\bW$. Under the LoRA parameterization, this problem is reformulated as \eqref{eq_strongly_convex_lora}, where the low rank factors  $(\bA,\bB)$ are the optimization variables.
Although the objective $\mL$ is strongly convex with respect to $\bW$, the problem \eqref{eq_strongly_convex_lora} is nonconvex in $(\bA,\bB)$ due to the bilinear structure of the parameterization.

We impose the following assumptions on the objective function $\mL$:
\begin{assumption}
\label{assumption1}
    The objective function in \eqref{eq_strongly_convex_lora} satisfies the following conditions:
    \begin{enumerate}
        \item The objective $\mL$ is $\mu$-strongly convex in $\bW$. Specifically, for any $\bW, \bW^{\prime}$ in the feasible region $\mathcal{W}$, we have $\mL(\bW)\geq \mL(\bW^{\prime}) + \left\langle \nabla_{\bW}\mL(\bW^{\prime}), \bW - \bW^{\prime}\right\rangle + \frac{\mu}{2}\left\| \bW - \bW^{\prime}\right\|_{\mathrm{F}}^2$.

        \item The objective $\mL$ is $M$-smooth in $\bW$. Specifically, for any $\bW, \bW^{\prime}$ in the feasible region $\mathcal{W}$, we have $\|\nabla_{\bW} \mL(\bW) - \nabla_{\bW} \mL(\bW^{\prime})\|_{\mathrm{F}} \leq M \left\| \bW - \bW^{\prime}\right\|_{\mathrm{F}}$.

        \item The objective $\mL$ is $L$-Lipschitz in $\bW$. Specifically, for any $\bW, \bW^{\prime}$ in the feasible region $\mathcal{W}$, we have $\|\mL(\bW) -  \mL(\bW^{\prime})\|_{\mathrm{F}} \leq L \left\| \bW - \bW^{\prime}\right\|_{\mathrm{F}}$.

        \item The unique minimizer $\bW^{\star}$ of the original problem \eqref{eq_strongly_convex} admits a low-rank representation of the form $\bW^{\star} = \bW_{\text{pt}} + \bB^{\star} \bA^{\star}$ for some $\bA^{\star}$ and $\bB^{\star}$, where $\bA^{\star} \in \mathbb{R}^{r \times n}$ and $\bB^{\star} \in \mathbb{R}^{m \times r}$.
    \end{enumerate}
\end{assumption}

While the factorization $(\bA^{\star}, \bB^{\star})$ is generally non-unique, the optimal full matrix $\bW^{\star}$ is uniquely determined by the strong convexity of $\mL$. 
We apply the ODELoRA dynamics in \eqref{eq13} to solve \eqref{eq_strongly_convex_lora}. We first define
$    \bW(t) = \bW_{\text{pt}} + \bB(t) \bA(t)$.
Differentiating with respect to $t$, we obtain
\begin{equation}
\label{eq17}
    \begin{split}
        & \frac{d\bW(t)}{dt} = \frac{d\bB(t)}{dt} \bA(t) + \bB(t) \frac{d\bA(t)}{dt}\\
        & = -\nabla_{\bW} \mL(t) + \left[\bI - \bB(t) \left(\bB(t)^{\top}\bB(t)\right)^{-1}\bB(t)^{\top}\right] \nabla_{\bW} \mL(t) \left[\bI - \bA(t)^{\top} \left(\bA(t) \bA(t)^{\top}\right)^{-1}\bA(t)\right]\\
        & := -\nabla_{\bW} \mL(t) +  \bP_{\bB(t)}^{\text{null}}\nabla_{\bW} \mL(t)\bP_{\bA(t)}^{\text{null}},
    \end{split}
\end{equation}
where  $\bP_{\bA(t)}^{\text{null}}$ and $\bP_{\bB(t)}^{\text{null}}$ denote the orthogonal projection matrices onto the null spaces of $\bA(t)$ and $\bB(t)$, respectively:
\begin{equation}
\label{eq21}
    \bP_{\bA(t)}^{\text{null}} = \bI - \bA(t)^{\top} \left(\bA(t) \bA(t)^{\top}\right)^{-1}\bA(t)
\end{equation}
and
\begin{equation}
\label{eq22}
    \bP_{\bB(t)}^{\text{null}} = \bI - \bB(t) \left(\bB(t)^{\top}\bB(t)\right)^{-1}\bB(t)^{\top}.
\end{equation}
Compared to full fine-tuning, \eqref{eq17} shows that ODELoRA follows the full gradient flow up to additional projection terms onto the null spaces induced by the low-rank factors. These projection components are an inherent consequence of the LoRA multiplicative structure and cannot, in general, be eliminated. Nevertheless, as will be shown in the subsequent analysis, under appropriate conditions these deviations do not prevent global convergence to the optimal solution $\bW^{\star}$.

\begin{theorem}
\label{theorem1}
    Suppose that Conditions 1 and 4 in \Cref{assumption1} hold for the feasible region $\mW = \left\{\bW = \bW_{\text{pt}} + \bB \bA: \bA\bA^{\top}~\text{and}~\bB^{\top}\bB~\text{are positive definite}\right\}$.
    Let $(\bA(t),\bB(t))$ be the continuous-time trajectory generated by \textsc{ODELoRA} (as in \eqref{eq13}) for solving \eqref{eq_strongly_convex_lora}. 
    Assume that for all $t \in \mathbb{R}_{+}$ and some $\varepsilon \in [0,1)$, $\bA(t)\bA(t)^{\top}$ and $\bB(t)^{\top}\bB(t)$ are always positive definite, and the following condition holds:
    \begin{equation}
    \label{eq18}
        \left\langle \bP_{\bB(t)}^{\text{null}}\nabla_{\bW} \mL(t)\bP_{\bA(t)}^{\text{null}}, \bW(t) - \bW^{\star}\right\rangle \leq \varepsilon \left\langle \nabla_{\bW} \mL(t), \bW(t) - \bW^{\star}\right\rangle.
    \end{equation}
    Then the trajectory $\bW(t)$ generated by \textsc{ODELoRA} converges linearly to $\bW^{\star}$, and the following holds:
    \begin{equation*}
        \left\| \bW(t) - \bW^{\star}\right\|_{\mathrm{F}}^2 \leq \exp\left(- \mu (1-\varepsilon) t\right) \left\| \bW(0) - \bW^{\star}\right\|_{\mathrm{F}}^2.
    \end{equation*}
\end{theorem}

The proof of \Cref{theorem1} is provided in \Cref{appendix1.1}. The condition in \eqref{eq18} requires that the projection-induced deviation in the ODELoRA dynamics be sufficiently controlled relative to the full gradient direction. In particular, when $\varepsilon \in [0,1)$, the dynamic $\frac{d \bW(t)}{dt}$ is dominated by the negative gradient $-\nabla_{\bW} \mL(t)$, ensuring that it remains a descent direction for $\bW(t)$. As a consequence, the trajectory $\bW(t)$ monotonically approaches the unique minimizer $\bW^{\star}$. Under this condition, ODELoRA can be viewed as an effective approximation of the gradient flow corresponding to full fine-tuning, up to a controlled perturbation arising from the low-rank constraint.

We then analyze the convergence of the discretized sequence $(\bA_t,\bB_t)$ generated by ODELoRA under various discretization schemes, including Euler, RK2, and RK4. Similarly, we denote $\bW_t := \bW_{\text{pt}} + \bB_t\bA_t$. We also define the projection matrices onto the null spaces of $\bA_t$ and $\bB_t$, as $\bP_{\bA_t}^{\text{null}}$ and $\bP_{\bB_t}^{\text{null}}$, respectively, analogous to \eqref{eq21} and \eqref{eq22}.
The following assumption characterizes the boundedness of the sequence $(\bA_t,\bB_t)$, thereby ensuring the numerical stability of the algorithm and guaranteeing the uniqueness of the solution $\bX_t$ to the Sylvester equation associated with $(\bA_t,\bB_t)$.

\begin{assumption}
\label{assumption2}
    Let $\{(\bA_t,\bB_t)\}_{t \ge 0}$ be the sequence generated by a time-discretized \textsc{ODELoRA} method using Euler, RK2, or RK4.
    We assume that the following boundedness and non-degeneracy conditions hold for all $t \ge 0$:
    there exist constants $0 < \gamma_{\min}\leq \gamma_{\max}$, such that $\gamma_{\min}\bI \preceq \bA_t\bA_t^{\top}\preceq\gamma_{\max}\bI$ and $\gamma_{\min}\bI \preceq \bB_t^{\top}\bB_t\preceq\gamma_{\max}\bI$.
    Moreover, all intermediate-stage iterates produced by the RK2 and RK4 schemes also satisfy the above conditions.

\end{assumption}

Under \Cref{assumption2}, we let the feasible region $\mathcal{W}$ in \Cref{assumption1} to be $\{\bW = \bW_{\text{pt}}+\bB\bA: \gamma_{\min}\bI \preceq \bA\bA^{\top}\preceq\gamma_{\max}\bI~\text{and}~\gamma_{\min}\bI \preceq \bB^{\top}\bB\preceq\gamma_{\max}\bI\}$.

In the next theorem, we extend \Cref{theorem1} to  discretized ODELoRA algorithms with Euler, RK2, and RK4 solvers, and show that these numerical schemes also enjoy linear convergence guarantees under appropriate choices of the step size.

\begin{theorem}
    \label{theorem2}
    Suppose that Assumptions \ref{assumption1} and \ref{assumption2} hold. 
    Let $\{(\bA_t,\bB_t)\}_{t\geq0}$ be the sequence generated by ODELoRA using Euler~\eqref{eq14}, RK2~\eqref{eq15}, and RK4~\eqref{eq16}. 
    Assume that
    \begin{equation}
    \label{eq19}
        \left\langle \bP_{\bB_t}^{\text{null}}\nabla_{\bW} \mL_{t}\bP_{\bA_t}^{\text{null}}, \nabla_{\bW} \mL_{t}\right\rangle \leq \varepsilon \left\|\nabla_{\bW} \mL_{t}\right\|_{\mathrm{F}}^{2},
    \end{equation}
    holds for all $t \in \mathbb{Z}_{+}$ and some $\varepsilon \in [0,1)$.
    Then the iterates generated by ODELoRA with the solver $\mathrm{sol} \in \{\mathrm{Euler},\mathrm{RK2},\mathrm{RK4}\}$ converge linearly in objective value. 
    In particular, there exists a constant $C^{\mathrm{sol}} > 0$ such that, if the step size satisfies $h^{\mathrm{sol}} \le C^{\mathrm{sol}}/M$, then
    \begin{equation}
    \label{eq23}
    \mathcal{L}_t^{\mathrm{sol}} - \mathcal{L}^{\star}
    \;\le\;
    \left(1 - (1-\varepsilon)\mu\, h^{\mathrm{sol}}\right)^{t}
    \left(\mathcal{L}_0^{\mathrm{sol}} - \mathcal{L}^{\star}\right).
    \end{equation}
    Here, the constant $C^{\mathrm{sol}}$ depends only (polynomially) on $L$, $\gamma_{\max}$, and $\gamma_{\min}$.
\end{theorem}

The proof of \Cref{theorem2} is deferred to \Cref{appendix1.2}.
The condition in \eqref{eq19} is essential for the convergence of discretized ODELoRA, as it quantifies the relative dominance of the true gradient direction within the column and row spaces spanned by $\bA_t$ and $\bB_t$. Owing to the bilinear structure of the LoRA parameterization, gradients may fall into the null spaces of these factors and thus be ineffective in driving progress toward the optimum.

Importantly, when $0\leq \varepsilon < 1$, condition \eqref{eq19} guarantees that a sufficient fraction of the gradient energy lies in the effective subspace, ensuring descent of the objective. In contrast, the extreme case $\varepsilon=1$ implies that the entire gradient lies in the null space of the LoRA parameterization, yielding $\frac{d\bW(t)}{dt} = 0$ and causing the algorithm to stall. This highlights the necessity of condition \eqref{eq19} for achieving meaningful convergence behavior.

Here, the constant $C^{\mathrm{sol}}$ depends only (polynomially) on $L$, $\gamma_{\max}$, and $\gamma_{\min}$, implying that $h = \mathcal{O}(\frac{1}{M})$.
Compared to the standard convergence analysis for strongly convex optimization problems, the step size in our result also scales linearly with the smoothness constant $M$, which is consistent with classical results~\cite{boyd2004convex}.
However, our analysis introduces additional dependencies on the Lipschitz constant $L$ as well as the boundedness parameters $\gamma_{\max}$ and $\gamma_{\min}$.
These additional dependencies arise from the multiplicative structure of the low-rank factors $\bA$ and $\bB$.
In particular, they are required to ensure the stability and well-conditioning of the factorized parameterization throughout the optimization process, so that the product $\bB_t\bA_t$ continues to provide a valid and stable representation of the full matrix.

\subsection{ODELoRA for Matrix Sensing}
In this subsection, we extend the results from the previous subsection to matrix sensing problems, which provide a canonical and analytically tractable setting for studying low-rank optimization dynamics. We consider the following matrix sensing problem:
\begin{equation*}
    \min_{\bW} \frac{1}{2}\left\| \bW \bS - \bY\right\|_{\mathrm{F}}^{2},
\end{equation*}
where $\bS \in \mathbb{R}^{n \times o}$ is the sensing matrix and $\bY\in \mathbb{R}^{m \times o}$ denotes the observed measurements. Under the LoRA fine-tuning framework, this problem is reformulated as
\begin{equation}
\label{eq24}
    \min_{\bA, \bB} \frac{1}{2}\left\| \left(\bW_{\text{pt}} + \bB \bA\right) \bS - \bY\right\|_{\mathrm{F}}^{2}.
\end{equation}
We assume that the measurements are noiseless and generated from a low-rank ground truth, i.e., $\bY = \left(\bW_{\text{pt}} + \bB^{\star} \bA^{\star} \right) \bS \in \mathbb{R}^{m \times o}$. Under this assumption, the objective can be equivalently written as
\begin{equation}
\label{eq26}
    \min_{\bA, \bB} \frac{1}{2}\left\| \left(\bB \bA - \bB^{\star} \bA^{\star}\right) \bX\right\|_{\mathrm{F}}^{2}.
\end{equation}

To characterize the geometry of this problem, we introduce the restricted isometry property (RIP), which is standard in matrix sensing and low-rank recovery.
\begin{definition}[Restricted Isometry Property]
    We say that the matrix $\bS$ satisfies the rank-$r$ restricted isometry property (RIP) with restricted isometric constant (RIC) $\delta \in [0,1)$, if
    for any matrix $\bW$ with rank less than $r$, the following inequality holds:
    \begin{equation*}
        (1-\delta)\left\|\bW \right\|_{\mathrm{F}}^2 \leq \left\|\bW \bS \right\|_{\mathrm{F}}^{2} \leq (1+\delta)\left\|\bW \right\|_{\mathrm{F}}^2.
    \end{equation*}
\end{definition}

Intuitively, the rank-$r$ RIP requires that the linear map $\bW \mapsto \bW\bS$ approximately preserves the Frobenius norm of all matrices with rank less than $r$. This condition ensures that low-rank matrices do not collapse or become excessively distorted after multiplication by $\bS$, so their geometric structure is preserved up to a multiplicative factor $\delta$. As a result, rank-$r$ RIP guarantees identifiability of low-rank solutions and provides a well-conditioned landscape for optimization, which is essential for the feasibility and theoretical analysis of matrix sensing and related low-rank recovery algorithms. Assumptions of this form are standard in matrix sensing and signal processing; see, e.g.,~\cite{candes2012exact,candes2010power,recht2010guaranteed}.

\begin{theorem}
\label{theorem3}
    Suppose that the sensing matrix $\bS$ satisfies the rank-$2r$ RIP with RIC $\delta \in [0,1)$, \Cref{assumption2} holds, and that
    \begin{equation}
    \label{eq27}
        \frac{\delta \frac{\sigma_{\max}(\bA^{\star})}{\sigma_{\min}(\bA^{\star})}  +  \frac{1}{\sqrt{1-\delta} \cdot \sigma_{\min}(\bA^{\star}) \sigma_{\min}(\bB^{\star})} \left\| \left(\bB_0\bA_0 - \bB^{\star}\bA^{\star}\right)\bS\right\|_{\mathrm{F}}}{1-\delta} := \varepsilon < 1.
    \end{equation}
    Then, ODELoRA and its discretizations (Euler, RK2, and RK4) achieve linear convergence with $\mu = 1-\delta$ and $M = 1+\delta$, i.e., \eqref{eq23} holds for \eqref{eq24}.
\end{theorem}

The proof follows directly from \Cref{theorem2}. In particular, the constant $\varepsilon$ results from computing the ratio $\left\langle \bP_{\bB_t}^{\text{null}}\nabla_{\bW} \mL_{t}\bP_{\bA_t}^{\text{null}}, \nabla_{\bW} \mL_{t}\right\rangle /  \left\|\nabla_{\bW} \mL_{t}\right\|_{\mathrm{F}}^{2}$  in \eqref{eq19}.
This ensures   the required dominance condition for the full-gradient component. As a result, both the continuous-time ODELoRA dynamics and its discretized counterparts (Euler, RK2, and RK4) enjoy linear convergence guarantees.

Compared with existing works on matrix sensing via low-rank factorization~\cite{xiong2024how,tu2016low}, the proposed ODELoRA-based methods (using Euler, RK2, and RK4 discretizations) are analyzed at the level of the full parameter matrix $\bW_t$. 
In contrast, prior analyses of classical gradient descent in matrix sensing predominantly study the optimization dynamics of the individual low-rank factors $(\bA_t,\bB_t)$. As a result, classical analyses characterize convergence indirectly through factor-level updates, whereas our analysis directly governs the dynamics of the full parameter matrix $\bW_t$.
This discrepancy arises from the fundamentally different training dynamics, where ODELoRA explicitly mimics the gradient flow of full fine-tuning while respecting the low-rank structure, whereas classical LoRA optimizes $(\bA,\bB)$ individually without exploiting their multiplicative coupling. 
This perspective enables a direct analysis of convergence in the full parameter space while respecting the low-rank structure.

If we specialize to the matrix factorization setting by taking $\bS = \bI$ and $\delta = 0$, the sensing operator becomes an exact isometry.
In this case, global linear convergence only requires a sufficiently good initialization:
\begin{equation*}
    \frac{1}{\sigma_{\min}(\bA^{\star}) \sigma_{\min}(\bB^{\star})} \left\| \bB_0\bA_0 - \bB^{\star}\bA^{\star}\right\|_{\mathrm{F}} := \varepsilon < 1,
\end{equation*}
in addition to a suitable step size.
This requirement is common in existing results on matrix factorization~\cite{park2018finding,wei2016guarantees,tu2016low}.

\subsection{Comparisons of Numerical Solvers}
\label{section4.3}
In general, the linear convergence results established above do not explicitly distinguish the advantages of RK4 over the Euler method. This phenomenon is well recognized in the optimization literature and highlights a fundamental difference between classical optimization analysis and numerical ODE analysis. For strongly convex objectives, the underlying continuous-time dynamics converge to a stationary point as $t \to \infty$, and linear convergence rates primarily characterize asymptotic behavior rather than fine-grained discretization accuracy. In contrast, numerical ODE analysis focuses on accurately approximating transient trajectories, where higher-order methods can offer substantial advantages.

To illustrate the numerical advantages of RK4 over Euler's method, we compare the discretization error bounds of \textsc{ODELoRA} obtained from classical numerical analysis of ODEs.
We assume that the objective function $\mathcal{L}$ is four-times continuously differentiable.
Under the boundedness condition in \Cref{assumption2}, the induced field $F$ is therefore also four-times continuously differentiable with respect to $\bW$.
For RK4 discretization~\cite{leveque2007finite,rosser1967runge,Hildebrand1987}, we obtain
\begin{equation*}
    \left\| \bW_{t} - \bW^{\star}\right\|_{\mathrm{F}} \leq \exp\left(- 0.5 (1-\varepsilon) \mu h^{\text{RK4}} t\right) \left\| \bW_{0} - \bW^{\star}\right\|_{\mathrm{F}} + \mathcal{O}\left(\left(h^{\text{RK4}}\right)^4\right),
\end{equation*}
whereas for the Euler method, we have
\begin{equation*}
    \left\| \bW_{t} - \bW^{\star}\right\|_{\mathrm{F}} \leq \exp\left(- 0.5 (1-\varepsilon) \mu h^{\text{Eur}} t \right) \left\| \bW_{0} - \bW^{\star}\right\|_{\mathrm{F}} + \mathcal{O}\left(h^{\text{Eur}}\right).
\end{equation*}
Here, we omit constant factors in the discretization error terms, which depend on the smoothness constants of $\mathcal{L}$ up to fourth-order derivatives.
Although both methods exhibit the same linear contraction structure in the leading term, RK4 enjoys a significantly smaller discretization error due to its higher-order accuracy. Consequently, there exists a threshold step size $\Bar{h}>0$ such that whenever $h^{\text{Euler}}, h^{\text{RK4}} \leq \Bar{h}$, the RK4 update enjoys a small error bound than the Euler update. This advantage becomes increasingly pronounced when accurate tracking of the continuous-time trajectory is desired.

In summary, while higher-order solvers such as RK4 do not improve the asymptotic linear convergence rate predicted by optimization theory, they provide superior numerical accuracy and stability in approximating the ODELoRA dynamics. This distinction explains why Runge--Kutta methods are often favored in practice, even when classical convergence guarantees appear similar.

\subsection{Stability and Efficiency of ODELoRA}
\label{section4.4}
In this section, we analyze the stability properties of discretized ODELoRA and its implications for stable feature learning. Our analysis follows the toy regression setup and the notion of stable feature learning introduced in \cite{hayou2024lora+}.

We consider the following regression problem:
\begin{equation*}
    \min_{\bW} \mL(\bW) := \left\| \bW\bs - \by\right\|^2,
\end{equation*}
where $(\bs,\by) \in \mathbb{R}^{n} \times \mathbb{R}^{m}$ denotes a single feature–label pair. Under the LoRA parameterization, this problem is reformulated as
\begin{equation}
\label{eq28}
    \min_{\bA,\bB} \mL(\bW_{\text{pt}}+ \bB\bA) = \left\| \left(\bW_{\text{pt}}+ \bB\bA\right)\bs - \by\right\|^2,
\end{equation}
where $\bW_{\text{pt}} \in \mathbb{R}^{m \times n}$ is the pretrained model parameter, $\bA \in \mathbb{R}^{r \times n}$, and $\bB \in \mathbb{R}^{m \times r}$ are trainable LoRA factors.
Let $\{(\bA_t,\bB_t)\}_{t\geq0}$ denote the sequence generated by ODELoRA with a numerical solver, such as Euler, RK2, or RK4, applied to \eqref{eq28}. Since Euler and RK2 can be viewed as simpler special cases, we focus on ODELoRA-RK4 for concreteness.
The corresponding one-step change in the model output along the input direction $\bs$ is given by
\begin{equation*}
    \begin{split}
        & \left(\bW_{\text{pt}}+ \bB_{t+1}\bA_{t+1}\right)\bs - \left(\bW_{\text{pt}}+ \bB_t\bA_t\right)\bs\\
        & = \left(\bB_{t+1}\bA_{t+1} - \bB_t\bA_t\right)\bs \\
        & = \frac{h}{6} F_{\bB}\left(\bA_t,\bB_t\right)\bA_t \bs + \frac{h}{6} \bB_t F_{\bA}\left(\bA_t,\bB_t\right)\bs + \frac{h}{3} F_{\bB}\left(\bA_t^{(1)},\bB_t^{(1)}\right)\bA_t \bs\\
        & \hspace{1em} + \frac{h}{3}\bB_t F_{\bA}\left(\bA_t^{(1)},\bB_t^{(1)}\right)\bs+ \frac{h}{3} F_{\bB}\left(\bA_t^{(2)},\bB_t^{(2)}\right)\bA_t \bs + \frac{h}{3}\bB_t F_{\bA}\left(\bA_t^{(2)},\bB_t^{(2)}\right)\bs\\
        & \hspace{1em} + \frac{h}{6}  F_{\bB}\left(\bA_t^{(3)},\bB_t^{(3)}\right)\bA_t \bs + \frac{h}{6}\bB_t F_{\bA}\left(\bA_t^{(3)},\bB_t^{(3)}\right)\bs\\
        & := \sum_{i=1}^{8} \varphi_{t}^{(i)},
    \end{split}
\end{equation*}
where $\{\varphi_{t}^{(i)}\}_{i=1}^{8}$ denotes the individual contributions from ODELoRA-RK4. 
Following \cite{hayou2024lora+}, stable feature learning is characterized by the relative scaling behavior of these update components with respect to the model size $n$. In particular, we impose the standard assumptions that $\left\|\bs\right\| = \Theta(1)$, $\left\|\bW_{\text{pt}} \bs - \by\right\| = \Theta(1)$, and that the output dimension satisfies $m = \Theta(n)$. Here, $\Theta\left(\cdot\right)$  means that the quantity of interest does not grow with $n$. More precisely, it denotes asymptotic scaling with respect to $n$, i.e., $f(n) = \Theta\left(g(n)\right)$ represents that $0 < \liminf_{n \to \infty} \big|\frac{f(n)}{g(n)} \big| \leq \limsup_{n \to \infty} \big|\frac{f(n)}{g(n)}\big| <+\infty$. 
Although these measured quantities may depend jointly on the model size $n$, the time variable $t$, and algorithm-dependent constants, we focus exclusively on their scaling behavior with respect to $n$, which is the dominant consideration in deep learning theory~\cite{jacot2018neural,lee2019wide}.

\begin{definition}
\label{definition1}
    We say that ODELoRA-RK4 achieves {\em stable feature learning} if $\big\|\varphi_{t}^{(i)}\big\| = \Theta(1)$, for all $i \in [8]$ and all $t \in \mathbb{N}_{+}$, where $\{\varphi_{t}^{(i)}\}_{i=1}^{8}$ denote the one-step update components defined above.
\end{definition}

\begin{remark}
The notion of stable feature learning for ODELoRA-RK4 is defined through the scaling behavior of the update components $\{\varphi_t^{(i)}\}_{i=1}^{8}$. This definition naturally extends to Euler and RK2 discretizations, which involve fewer intermediate update terms. Since RK4 introduces additional intermediate stages and therefore imposes stronger stability requirements, we restrict attention to the RK4 case here.

This condition ensures that, after a single update step, the change in the model output remains a bounded function of  the model size $n$. To see this, suppose that a one-step update component scales as $\Theta(n^{\alpha})$. For stable training, it must hold that $\alpha \leq 0$, otherwise, the model output grows unboundedly with increasing $n$, leading to numerical instability and training failure. On the other hand, if $\alpha < 0$, the magnitude of the update vanishes as the model size increases, resulting in excessively slow training in the large model regime. Therefore, to achieve both stability and efficiency, we require the critical scaling regime in which $\alpha = 0$, as stated in Definition~\ref{definition1}.
\end{remark}

We assume that the LoRA parameters are initialized such that the rows of $\bA_0$ satisfy $\left\|\bA_0(j,:) \right\| = \Theta(1)$ and $\left\|\bA_0\bs\right\| = \Theta(1)$, for $j \in [r]$, where $\bs$ denotes the input feature vector.
For $\bB_0$, we initialize it as a zero matrix, and after a step of updates, both $\bA$ and $\bB$ can be balanced to satisfy $\mathcal{M}$. 
This assumption is consistent with the widely adopted practice of initializing LoRA using the dominant singular vectors of the pre-trained weight matrix $\bW_{\mathrm{pt}}$~\cite{balazy2024lora,liu2024dora,meng2024pissa,lingam2024svft}. Empirical evidence suggests that pre-trained weights capture the principal subspace of each layer, so projections onto these dominant singular directions preserve both parameter scale and signal magnitude~\cite{balazy2024lora,liu2024dora,meng2024pissa,lingam2024svft,zhang2025loraone}. Consequently, the above conditions naturally hold for such singular-vector-based initializations.

\begin{theorem}
\label{theorem4}
ODELoRA with Euler, RK2, or RK4 discretization, achieves stable feature learning (\Cref{definition1}) with a constant step size $h = \Theta(1)$, and the constant does not depend on the model size $n$.
\end{theorem}

The detailed proof is provided in \Cref{appendix2}. As pointed out in LoRA+~\cite{hayou2024lora+}, classical LoRA fails to achieve stable feature learning with a constant step size; instead, the admissible step size must scale with the model size $n$. Such sensitivity to model dimension is undesirable in deep learning, where robustness with respect to architectural scaling is crucial.
In contrast, \Cref{theorem4} shows that the proposed ODELoRA framework achieves stable feature learning with a constant step size $h = \Theta(1)$. This implies that the permissible step size remains essentially unchanged as the model size grows, showing improved robustness and scalability.


\section{Experiments}
\label{section5}
\subsection{Toy Examples on Matrix Sensing}

\begin{figure}[tb!]
    \centering
    \subfigure[$\delta=0.05$, $h=1$]{%
        \includegraphics[width=0.30\textwidth]{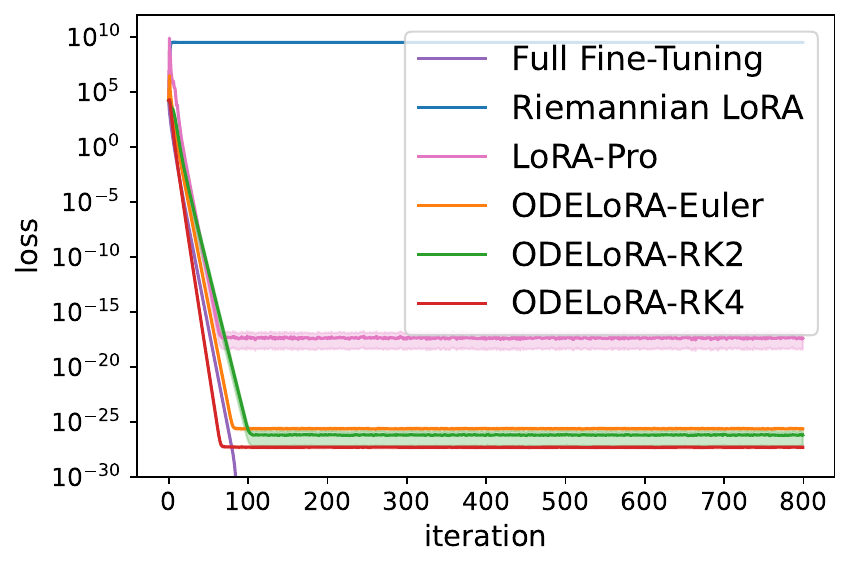} \label{fig_loss_iteration_del_0.05_h_1.00}}
    \subfigure[$\delta=0.05$, $h=0.5$]{%
        \includegraphics[width=0.30\textwidth]{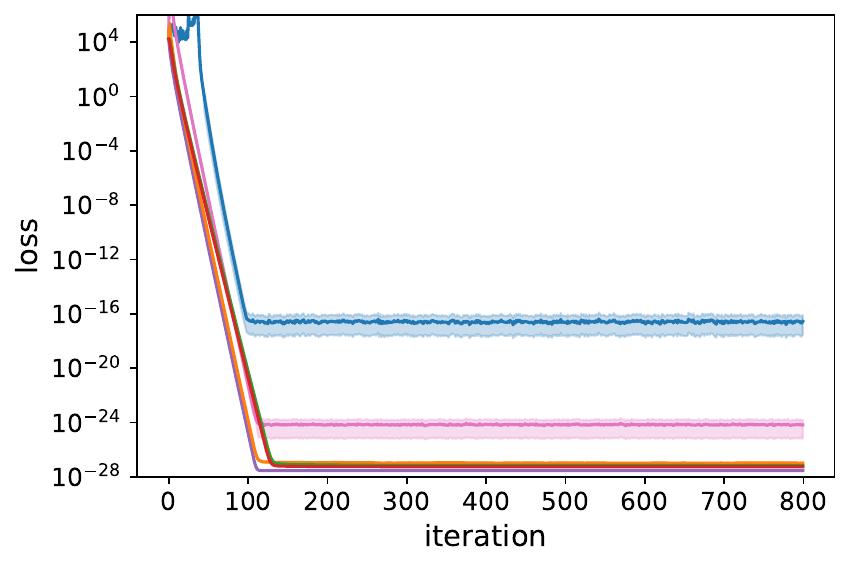} \label{fig_oss_iteration_del_0.05_h_0.50}}
    \subfigure[$\delta=0.05$, $h=0.1$]{%
        \includegraphics[width=0.30\textwidth]{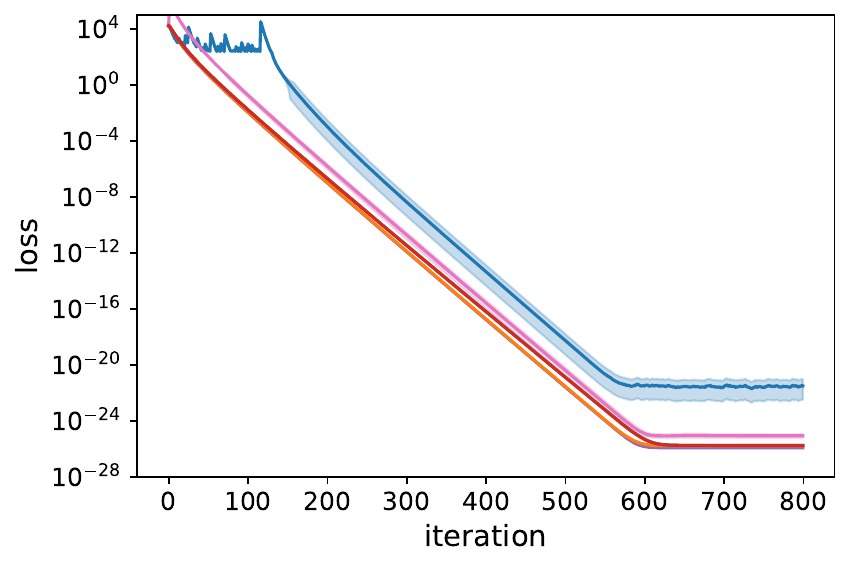} \label{fig_oss_iteration_del_0.05_h_0.10}}\\
    \subfigure[$\delta=0.1$, $h=0.8$]{%
        \includegraphics[width=0.30\textwidth]{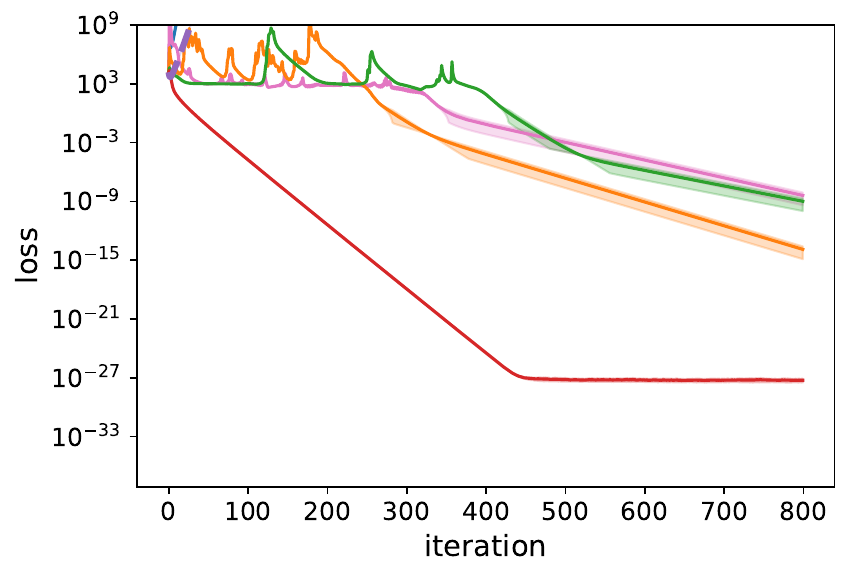} \label{fig_loss_iteration_del_0.1_h_0.80}}
    \subfigure[$\delta=0.1$, $h=0.7$]{%
        \includegraphics[width=0.30\textwidth]{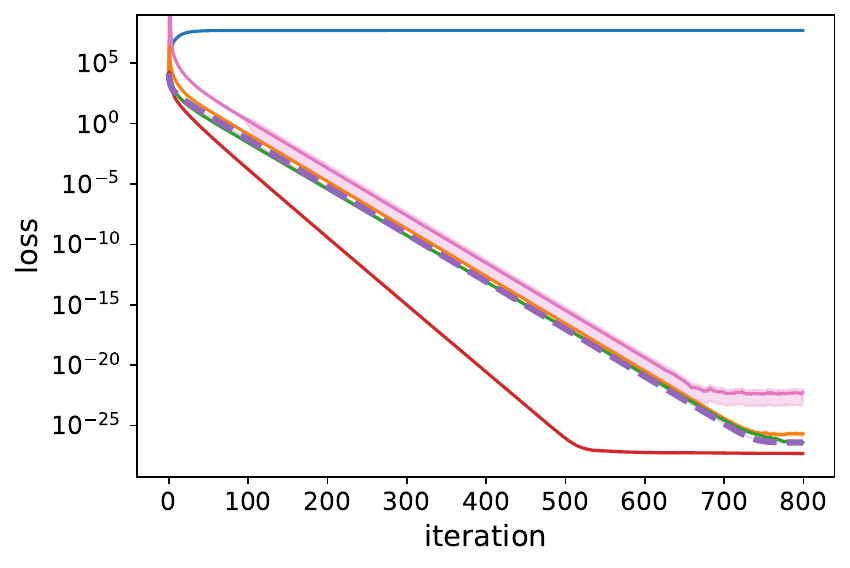} \label{fig_loss_iteration_del_0.1_h_0.70}}
    \subfigure[$\delta=0.1$, $h=0.1$]{%
        \includegraphics[width=0.30\textwidth]{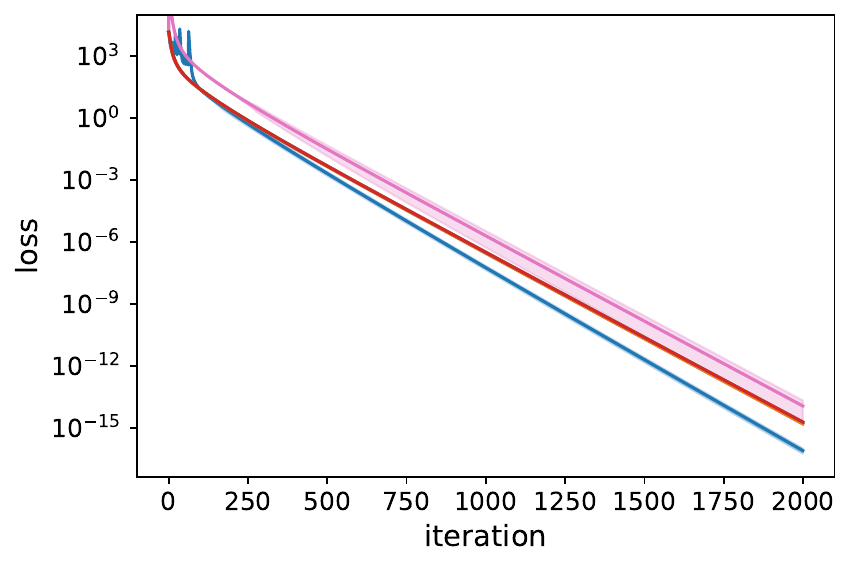} \label{fig_loss_iteration_del_0.1_h_0.10}}
    \caption{Performance comparisons of full fine-tuning, Riemannian LoRA, LoRA-Pro, and ODELoRA with Euler and Runge--Kutta discretizations on matrix sensing problems under varying RIC $\delta$ and step size $h$. }
    \label{fig_loss_iteration_del_h}
\end{figure}

We conduct numerical experiments on the matrix sensing problem~\eqref{eq26} to validate our theoretical results on the linear convergence of ODELoRA.

The trajectories of the objective values for different optimization methods, including full fine-tuning, Riemannian LoRA, LoRA-Pro, and ODELoRA with Euler and Runge--Kutta (RK2 and RK4) discretizations, under varying RIC $\delta$ and step sizes $h$ are shown in \Cref{fig_loss_iteration_del_h}. 
In all these cases, the classical LoRA with vanilla gradient descent is divergent under these step sizes. 
As illustrated in the figures, whenever the methods converge, ODELoRA with different discretizations exhibits linear convergence with essentially the same rate, which is consistent with our theoretical analysis.
When larger step sizes are applied, ODELoRA-RK4 remains stable and convergent, whereas RK2 and Euler become unstable and diverge. This highlights the advantage of higher-order discretization methods, which provide improved numerical stability and robustness to larger step sizes. These observations are consistent with the analysis in \Cref{section4.3}, where higher-order methods are shown to incur smaller discretization errors than lower-order methods under the same step size.
Moreover, it is surprising that ODELoRA-RK4 exhibit stable convergence under the setting $(\delta,h) = (0.1,0.8)$, outperforming all other methods, including full fine-tuning. We attribute this behavior to a combination of several reasons: the LoRA parameterization that explicitly constrains the low-rank structure; the proposed balanced manifold constraints, which further regularize the optimization trajectory; as well as the high-order accuracy of the RK4 integrator, which provides a more powerful and accurate discretization of the underlying continuous optimization flow.
A similar phenomenon is observed for $(\delta,h) = (0.1,0.7)$, where ODELoRA-RK4 converges faster than full fine-tuning and other discretization schemes, highlighting the robustness and efficiency of higher-order ODE-based optimization in challenging regimes.
Overall, these results demonstrate the benefit of Runge--Kutta methods for discretizing ODELoRA.
Although higher-order schemes such as RK4 require a higher per-iteration computational cost, they admit significantly larger step sizes and provide superior stability. 
As a result, ODELoRA-RK4 is particularly suitable for precise fine-tuning in the later stages of optimization.

We also observe that Riemannian LoRA exhibits instability and oscillatory behavior during the early stages of training. In contrast, ODELoRA demonstrates consistent and monotonic improvement throughout the optimization process.
This difference can be explained by the underlying optimization mechanisms. Riemannian LoRA can be interpreted as a preconditioned gradient descent method on the factorized variables $(\bA,\bB)$, where the initial preconditioner may be inaccurate or overly aggressive. As training progresses and the magnitudes of $(\bA,\bB)$ become appropriately scaled, the preconditioning effect improves, resulting in more stable updates.
Moreover, Riemannian LoRA deviates more substantially from full fine-tuning compared to ODELoRA, although it may eventually reach comparable or even faster convergence in certain settings. 
For LoRA-Pro, which mimics full fine-tuning without balanced manifold regularization, we observe that the loss initially increases at the beginning of training. This transient behavior can be attributed to imbalanced iterates in the low-rank parameterization; as the optimization progresses and a certain equilibrium is established, the loss subsequently decreases and the method begins to converge.
Taken together, these observations further support the close alignment between ODELoRA and full fine-tuning and the training stability of ODELoRA. In particular, the formulation of ODELoRA in \eqref{eq31}, which explicitly accounts for matching and balance, provides an explanation for its superior behavior. In contrast, Riemannian LoRA can be interpreted as a preconditioned optimization method, effectively implementing an adaptive step-size gradient descent scheme, which leads to qualitatively different optimization dynamics.

\subsection{Physics-Informed Neural Networks}

In well-structured matrix sensing problems, ODELoRA with different discretization schemes exhibits similar optimization trajectories when relatively small step sizes are used, and the advantage of higher-order methods such as RK4 is less obvious. This behavior can be attributed to the favorable properties of matrix sensing objectives, which are well-conditioned and strongly convex. Consequently, the stability benefits of ODELoRA-RK4 do not manifest prominently in such settings.
To further highlight the advantages of the proposed method, we evaluate ODELoRA on the fine-tuning of physics-informed neural networks (PINNs), where the loss landscape is significantly more complicated and nonconvex due to the involvement of differential operators. In this regime, optimization stability becomes more critical, allowing the benefits of higher-order discretization schemes to be more clearly observed.

We follow the PINN formulation in~\cite{raissi2019physics}, where solutions to differential equations are represented by multi-layer perceptrons (MLPs). 
When fine-tuning PINNs for a new PDE task, instead of updating all model parameters, we adopt LoRA-based fine-tuning, which substantially reduces the number of trainable parameters~\cite{wang2025transfer}. 
In all experiments, we use MLPs with a hidden dimension of $1000$ and hyperbolic tangent activation functions. The LoRA rank for all parameter matrices is fixed to $r=4$ throughout the paper.
We compute the relative error (rel) between the predicted solution $\phi(\bx;\Theta)$ and the ground truth solution $\bu(\bx)$ on a test dataset $\{(\bx_i, \bu(\bx_i))\}_{i=1}^{N_t}$, defined as $\sqrt{\frac{\sum_{i=1}^{N_t} \|\phi(\bx_i;\Theta) - \bu(\bx_i)\|^2}{\sum_{i=1}^{N_t} \|\bu(\bx_i)\|^2}}$, where $\phi(\cdot;\Theta)$ is the neural network parameterized by $\Theta$.

\subsubsection{Allen-Cahn Equations}

\begin{figure}[tb!]
    \centering
    \subfigure[$\blam=(1,1)$, training loss]{%
        \includegraphics[width=0.30\textwidth]{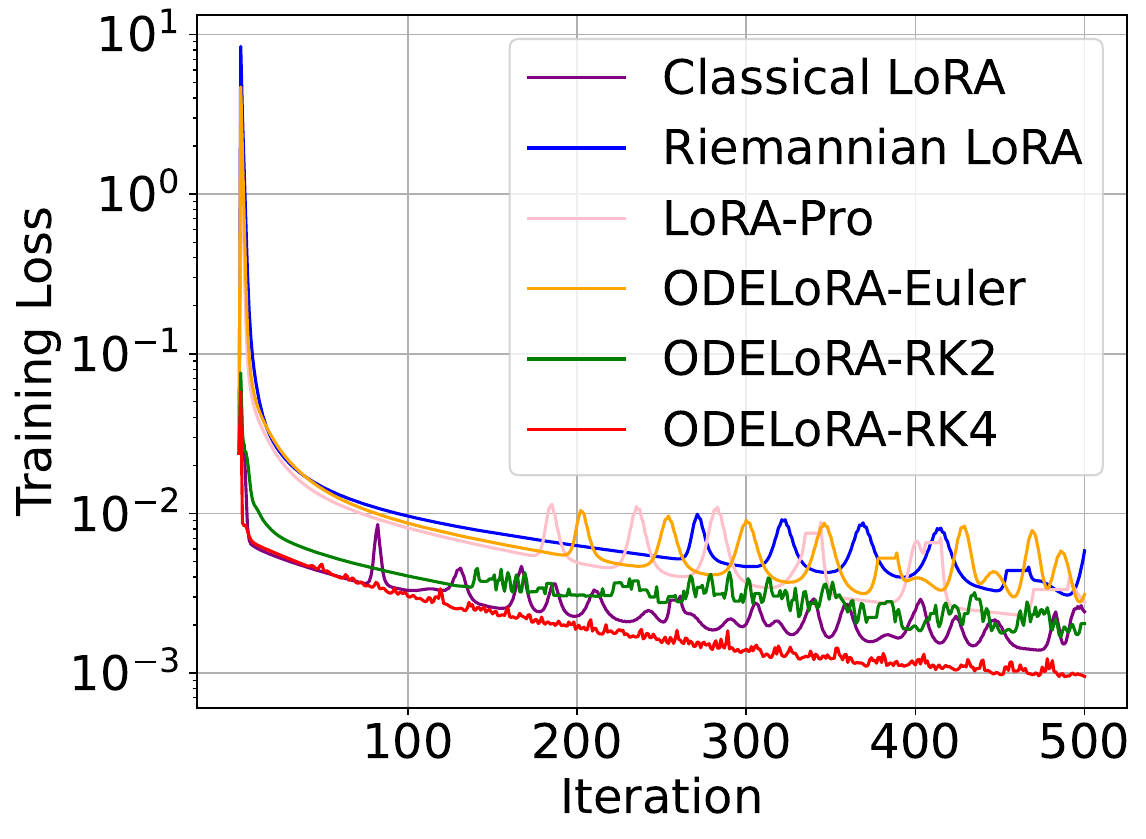} \label{fig_allen-cahn_1_1_loss}}
    \subfigure[$\blam=(1,5)$, training loss]{%
        \includegraphics[width=0.30\textwidth]{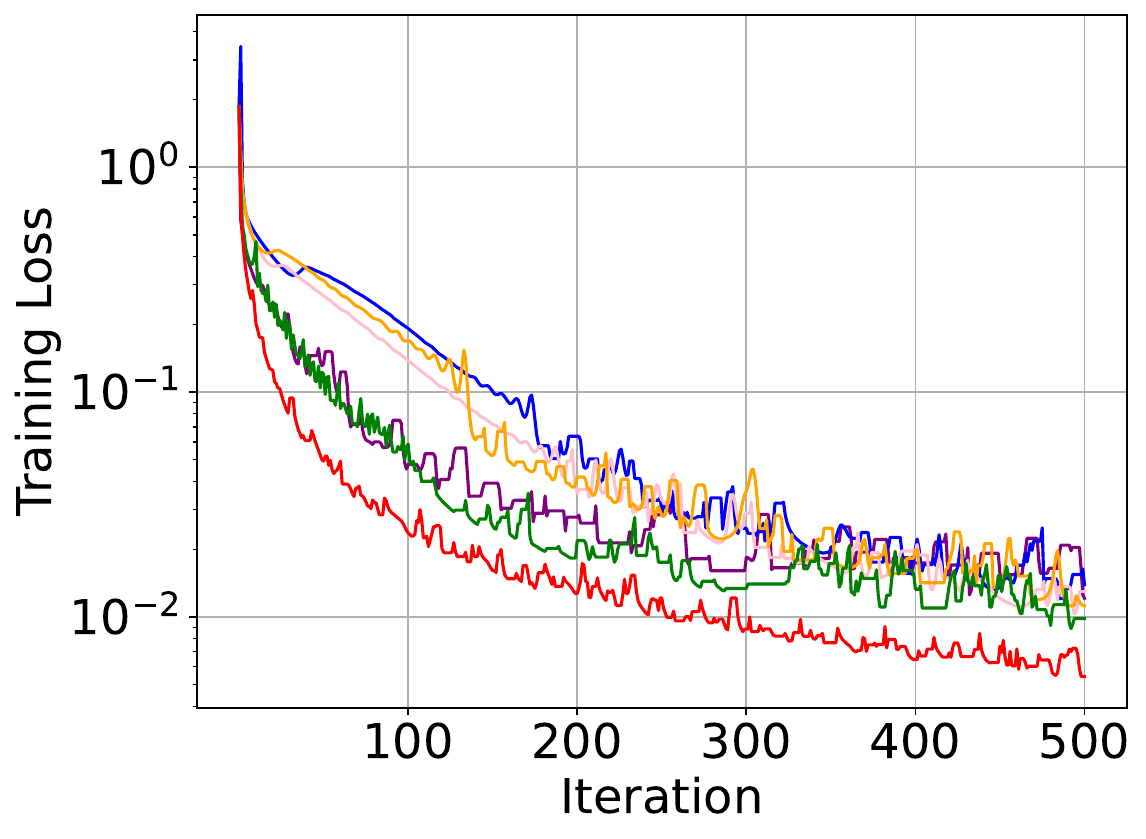} \label{fig_allen-cahn_1_5_loss}}
    \subfigure[$\blam=(2,1)$, training loss]{%
        \includegraphics[width=0.30\textwidth]{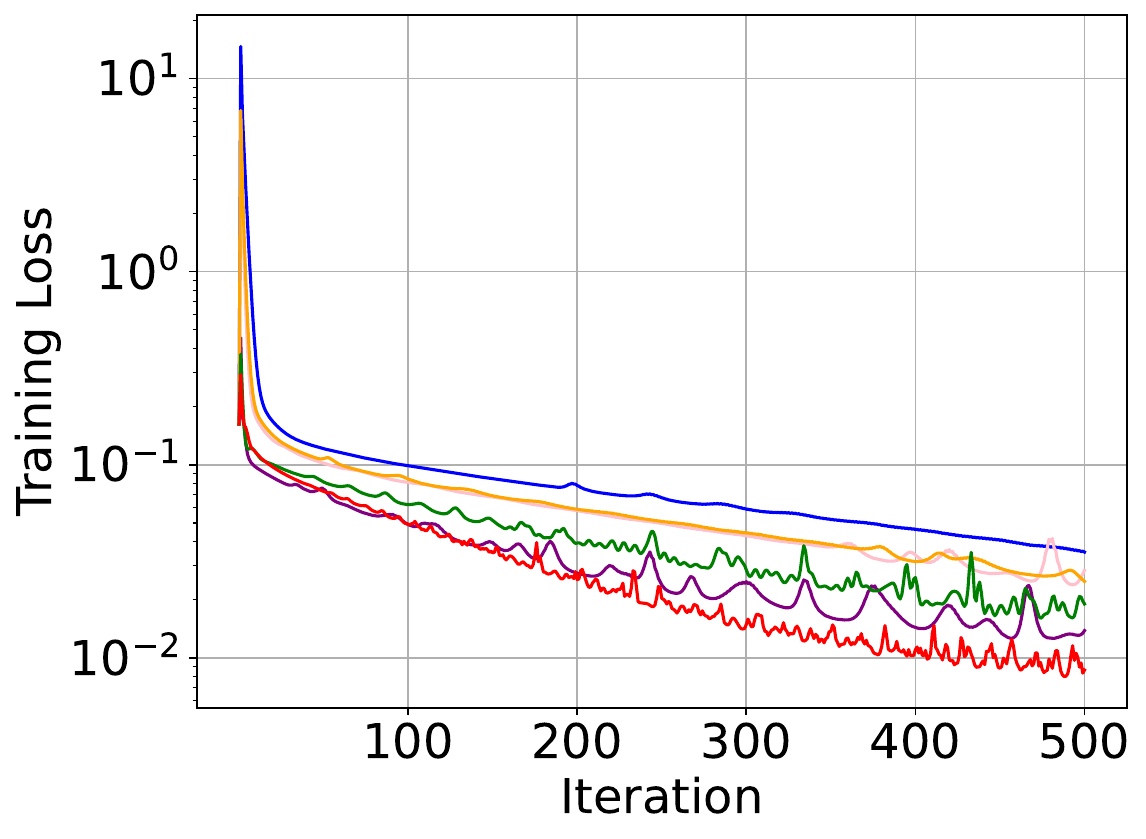} \label{fig_allen-cahn_2_1_loss}}\\
    \subfigure[$\blam=(1,1)$, relative error]{%
        \includegraphics[width=0.30\textwidth]{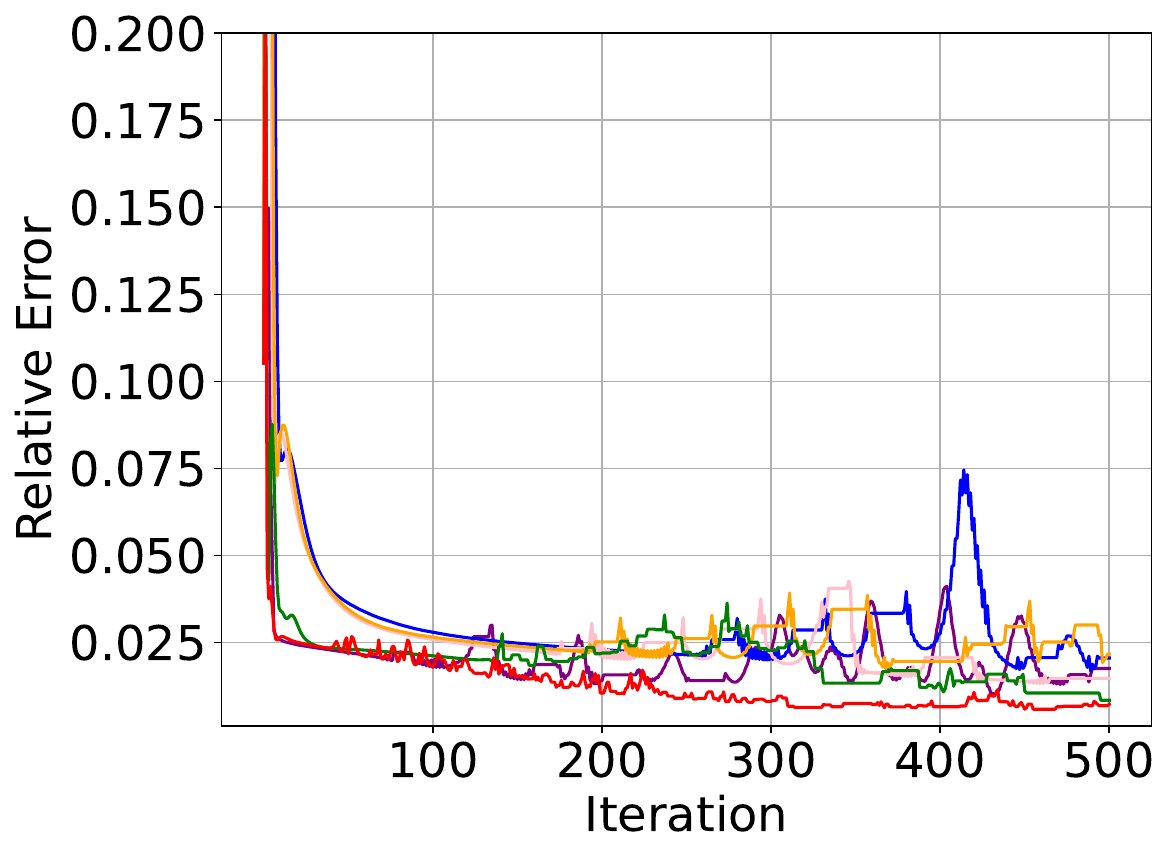} \label{fig_allen-cahn_1_1_error}}
    \subfigure[$\blam=(1,5)$, relative error]{%
        \includegraphics[width=0.30\textwidth]{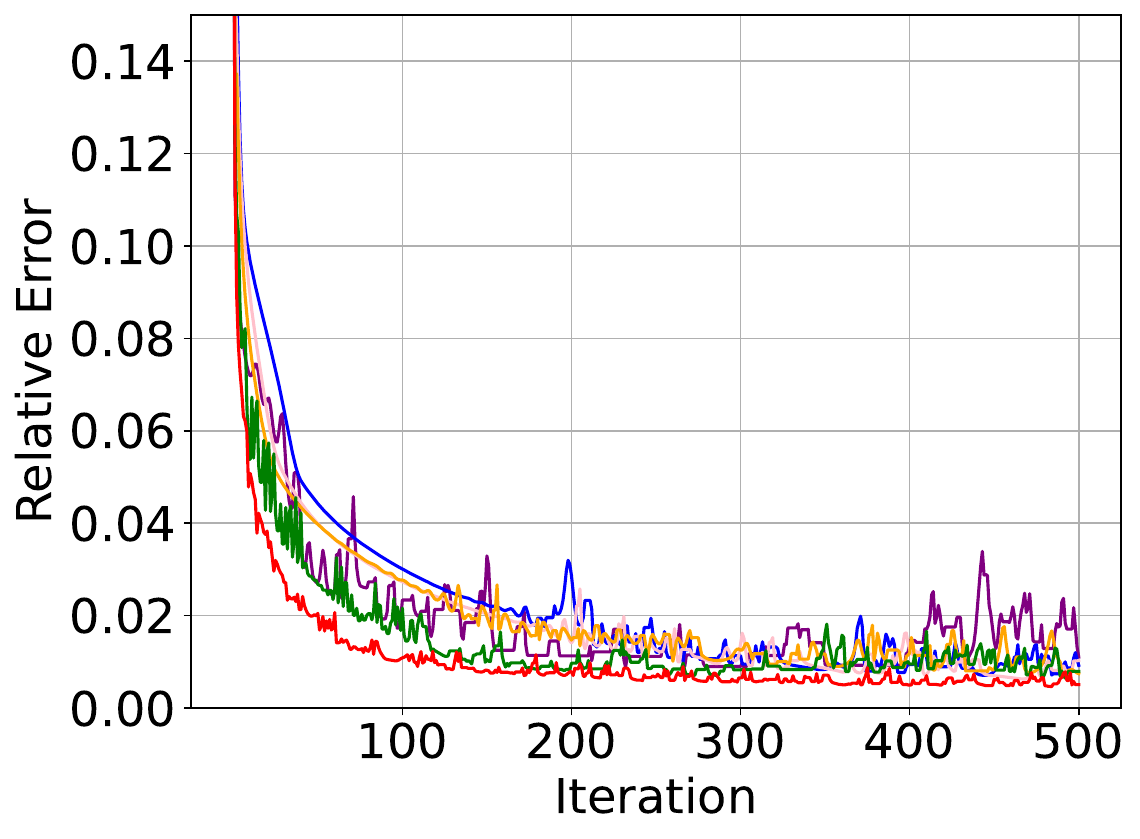} \label{fig_allen-cahn_1_5_error}}
    \subfigure[$\blam=(2,1)$, relative error]{%
        \includegraphics[width=0.30\textwidth]{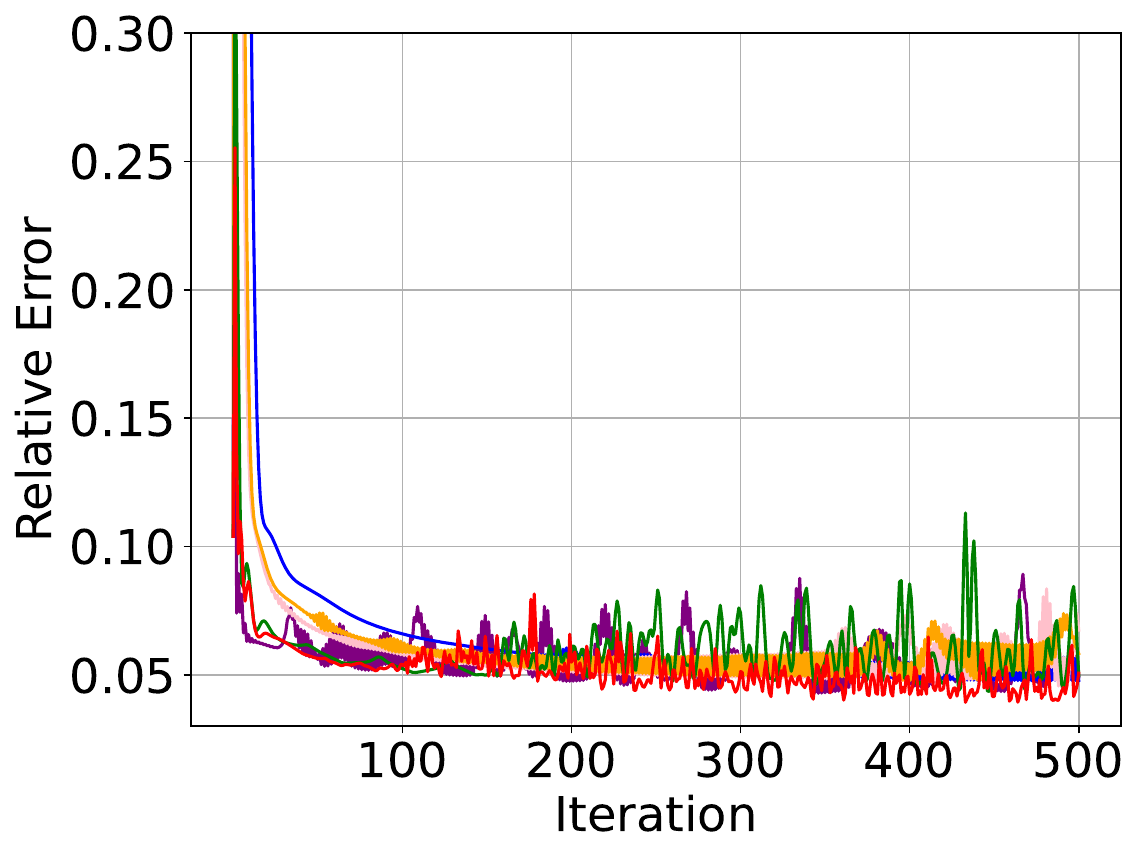} \label{allen-cahn_2_1_error}}
    \caption{Performance comparison of classical LoRA, Riemannian LoRA, LoRA-Pro, and ODELoRA with Euler and Runge--Kutta discretizations on solving Allen-Cahn equations under varying physical parameters $\bm{\lambda}$. }
    \label{fig_allen-cahn_loss_error}
\end{figure}

We consider the following Allen-Cahn equations with varying physical parameters $\blam \in \mathbb{R}^2$:
\begin{equation}
    \begin{split}
        \frac{\partial u(t,\bm{x};\blam)}{\partial t} & - \Delta u(t,\bm{x};\blam) -u(t,\bm{x};\blam)^3 + u(t,\bm{x};\blam) = g(t,\bm{x};\blam), \quad (t,\bm{x}) \in [0,1] \times \Omega,\\
        u(t,\bm{x};\blam) & = h_{1}(t,\bm{x};\blam), \quad (t,\bm{x}) \in [0,1] \times \partial \Omega,\\
        u(0,\bm{x};\blam) & = h_{2}(\bm{x};\blam), \quad \bm{x} \in \Omega,
    \end{split}
\end{equation}
where the temporal and  spatial domains are defined as $[0,1]$ and $\Omega = \{\bx \in \mathbb{R}^{2}: \left\|\bx \right\| \leq 1\}$, respectively. The exact solution $u(\bx;\blam)$ with the parameter $\blam$ is defined as $u(\bx;\blam) = e^{-t} \sin \left(\frac{\pi \lambda_1}{2}\left( 1 - \left\|\bm{x} \right\|\right)^{2.5} \right) + \lambda_2 \cdot e^{-t} \sin \left(\frac{\pi}{2}\left( 1 - \left\|\bm{x} \right\|\right) \right)$. 
Here, $\lambda_1$ represents the frequency, and $\lambda_2$ controls the variation level. 
We first pre-train a MLP model on the PDE with $\blam =(1,0)$, and then fine-tune it by LoRA on PDEs with $\blam=(1,1)$, $(1,5)$, and $(2,1)$.

We visualize the experimental results in \Cref{fig_allen-cahn_loss_error}, which reports both the training loss and the relative error under different physical parameters~$\blam$.
Overall, we observe that among all compared optimizers, ODELoRA-RK4 exhibits the most stable training dynamics, characterized by a smooth and consistently decreasing trajectory across all settings. In particular, ODELoRA-RK4 shows significantly reduced oscillation, indicating strong numerical stability during optimization.
Furthermore, ODELoRA-RK2 also demonstrates improved stability compared to classical LoRA, Riemannian LoRA, LoRA-Pro, and ODELoRA-Euler, although its convergence is slightly less smooth than that of RK4. These results support our main claim and motivation that higher-order discretization schemes lead to more stable optimization trajectories, as they better approximate the underlying continuous-time optimization flow induced by the proposed ODE formulation.
These results are fully consistent with classical results from numerical ODE solvers.
In contrast, classical LoRA exhibits apparent oscillatory and zigzag behavior, especially in the later stages of training. This instability can be attributed to the fact that classical LoRA optimizes the low-rank factors individually using standard gradient descent, without explicitly accounting for their multiplicative structures. As a result, gradient updates on individual factors can amplify instability due to the bilinear structure of the LoRA parameterization.
For Riemannian LoRA, although geometric constraints are imposed through gradient preconditioning on the low-rank manifold, the training dynamics still remain relatively unstable. This suggests that while Riemannian methods incorporate geometric information, they do not explicitly model the coupled continuous-time dynamics of the LoRA factors, and thus may still deviate from the ideal optimization flow.

When considering computational efficiency, the first four methods, classical LoRA, Riemannian LoRA, LoRA-Pro, and ODELoRA-Euler, are first-order methods, requiring only a single gradient evaluation per update. However, as observed in \Cref{fig_allen-cahn_loss_error}, their convergence is consistently slower than that of ODELoRA-RK2 and ODELoRA-RK4.
Although ODELoRA-RK4 requires approximately four times the per-iteration computational cost of classical LoRA, Riemannian LoRA, LoRA-Pro, and ODELoRA-Euler due to multiple gradient evaluations at intermediate stages, its convergence speed measured in terms of computational cost remains highly competitive, in addition to its superior stability. For example, when comparing the training loss and relative error at iteration $500$ for first-order methods, we consider the performance of ODELoRA-RK4 at iteration $100$ to ensure a fair comparison in terms of total gradient evaluations. Even under this conservative comparison, ODELoRA-RK4 achieves comparable, and in many cases, better performance.
These results indicate that ODELoRA-RK4 maintains a competitive convergence rate even after accounting for its higher per-iteration cost, while simultaneously providing significantly improved numerical stability. Such stability is notably absent in classical LoRA, which exhibits oscillations throughout training.

Finally, these observations suggest a practical hybrid optimization strategy for LoRA fine-tuning. In the early stages of training, fast methods such as LoRA-Pro or ODELoRA-Euler may be adopted to achieve rapid progress. As the optimization approaches convergence, one can transition to ODELoRA-RK4 to benefit from its enhanced stability and accelerated convergence due to its close tracking of the continuous optimization flow mimicing the full fine-tuning.

\subsubsection{Elliptic Equations}

\begin{figure}[tb!]
    \centering
    \subfigure[$\blam=(1,1)$, training loss]{%
        \includegraphics[width=0.30\textwidth]{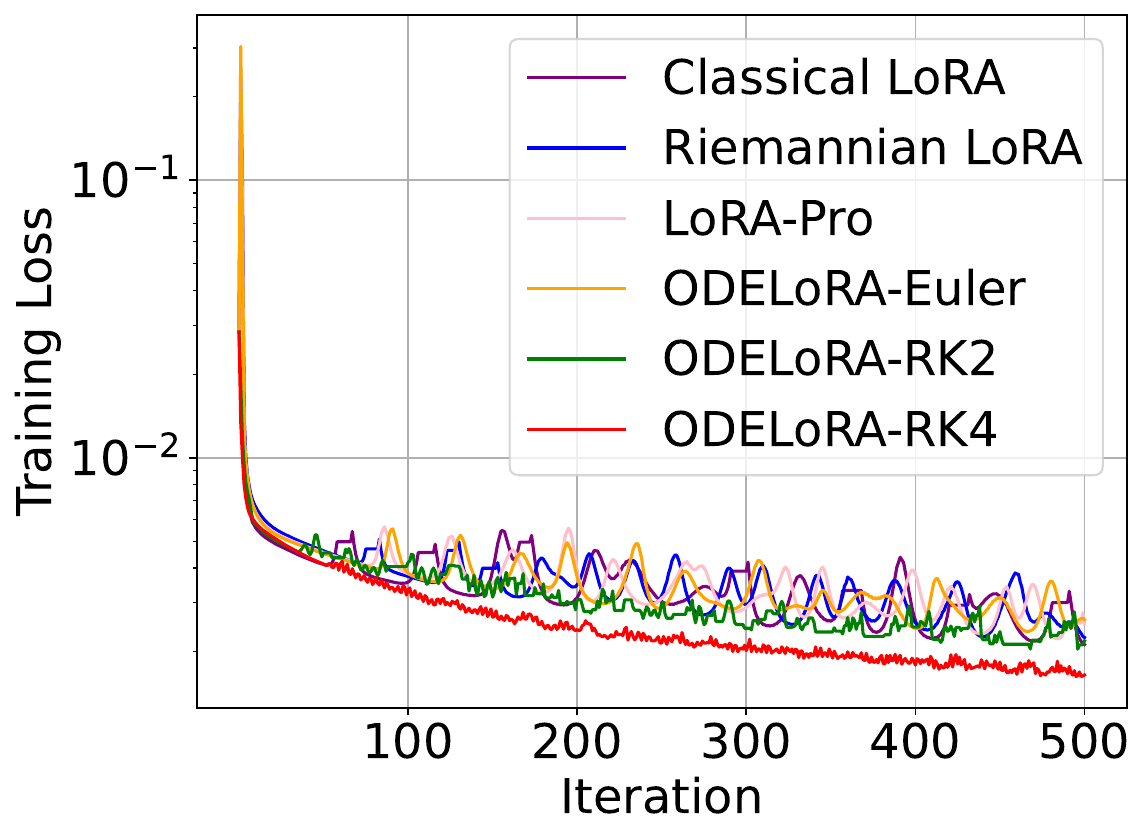} \label{fig_elliptic_1_1_loss}}
    \subfigure[$\blam=(1,5)$, training loss]{%
        \includegraphics[width=0.30\textwidth]{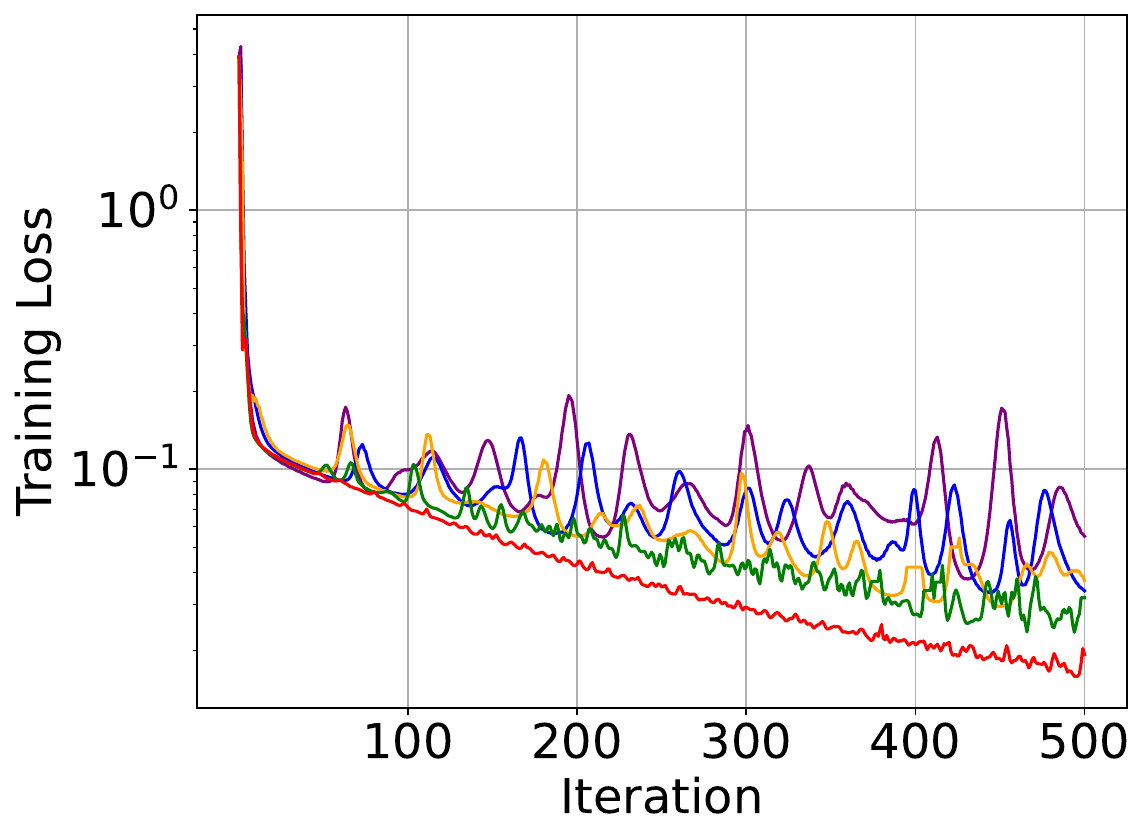} \label{fig_elliptic_1_5_loss}}
    \subfigure[$\blam=(2,1)$, training loss]{%
        \includegraphics[width=0.30\textwidth]{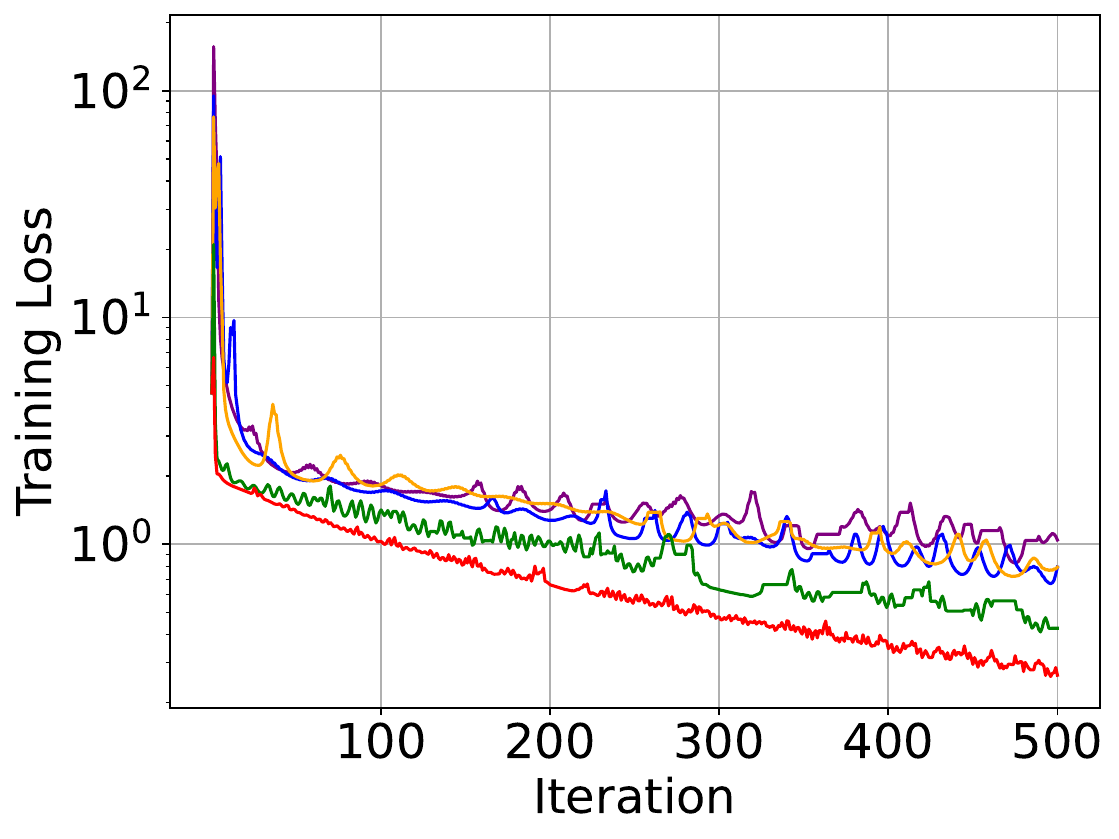} \label{fig_elliptic_2_1_loss}}\\
    \subfigure[$\blam=(1,1)$, relative error]{%
        \includegraphics[width=0.30\textwidth]{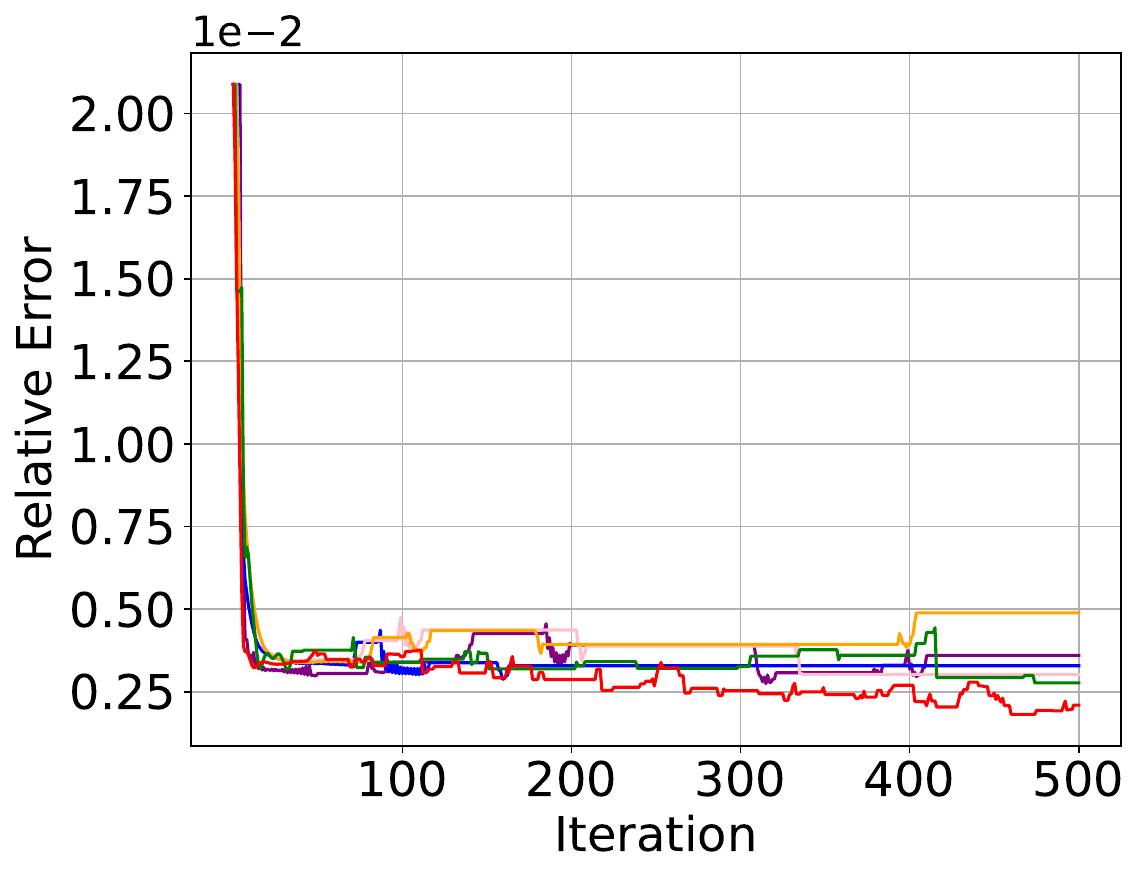} \label{fig_elliptic_1_1_error}}
    \subfigure[$\blam=(1,5)$, relative error]{%
        \includegraphics[width=0.30\textwidth]{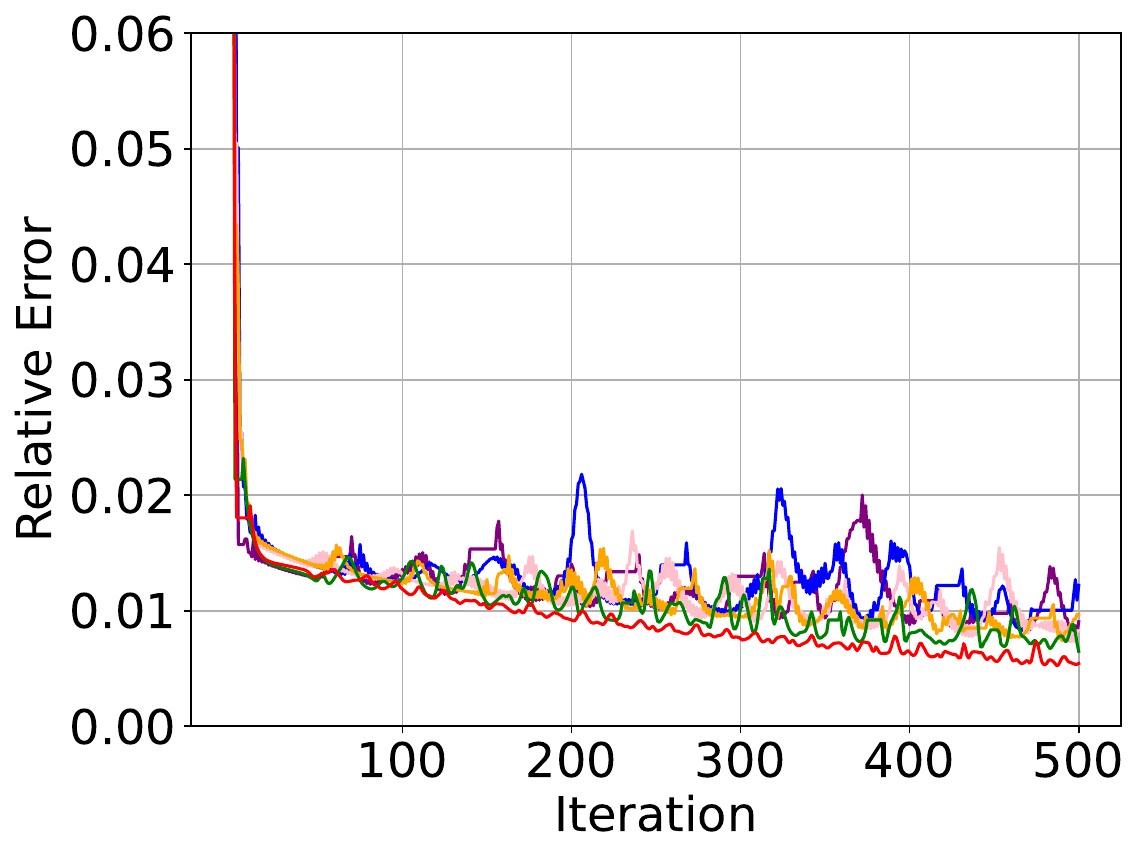} \label{fig_elliptic_1_5_error}}
    \subfigure[$\blam=(2,1)$, relative error]{%
        \includegraphics[width=0.30\textwidth]{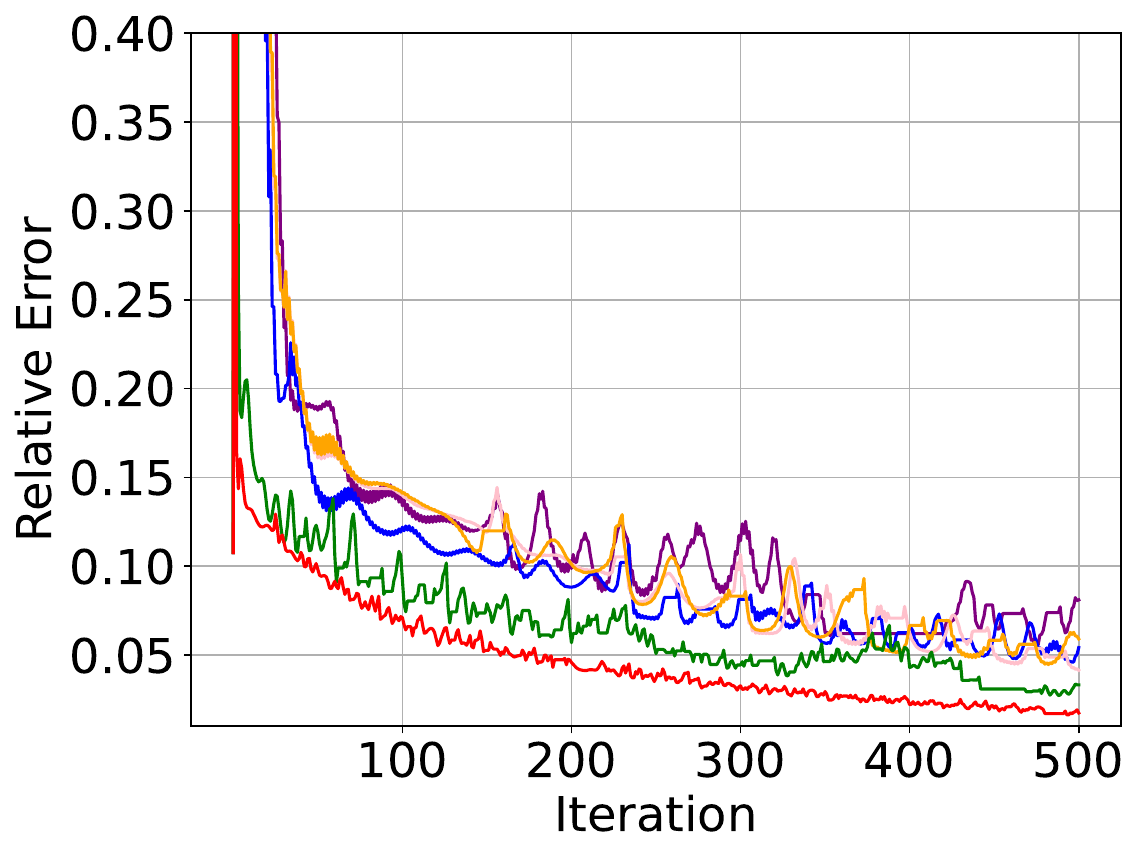} \label{fig_elliptic_2_1_error}}
    \caption{Performance comparison of classical LoRA, Riemannian LoRA, LoRA-Pro, and ODELoRA with Euler and Runge--Kutta discretizations on solving elliptic equations under varying physical parameters $\bm{\lambda}$. }
    \label{fig_elliptic_loss_error}
\end{figure}

We consider a class of elliptic equations with varying physical parameters $\blam \in \mathbb{R}^2$:
\begin{align}
        - \nabla \cdot \left(a(\bm{x}) \cdot \nabla u(\bm{x};\blam)\right) + \left\| \nabla u(\bm{x};\blam)\right\|^2   = g(\bm{x};\blam), \quad \bm{x} &\in \Omega, \nonumber\\*
        u(\bm{x};\blam) = h(\bm{x};\blam), \quad \bm{x} &\in \partial \Omega,
       \end{align} 
where the domain is defined as $\Omega = \{\bx \in \mathbb{R}^{2}: \left\|\bx \right\| \leq 1\}$, and the coefficient function is given by $a(\bx) = 1 + \frac{1}{2}\left\| \bx\right\|^2$. The exact solution $\bu(\bx;\blam)$ with the parameter $\blam$ is defined as $u(\bx;\blam) = \sin \big(\frac{\pi \lambda_1}{2}\big( 1 - \left\|\bm{x} \right\| \big)^{2.5} \big) + \lambda_2 \cdot \sin \big(\frac{\pi}{2}\big( 1 - \left\|\bm{x} \right\| \big) \big)$.

We fine-tune PINNs on a collection of elliptic equations using LoRA with different optimizers for updating the low-rank factors, and report the corresponding training loss and relative error in \Cref{fig_elliptic_loss_error}.
The behavior across different optimizers is consistent with the observations for the Allen–Cahn equation discussed in the previous subsection.
However, in this setting, we observe more obvious instability for classical LoRA, Riemannian LoRA, LoRA-Pro, and ODELoRA-Euler, even when relatively small step sizes $h$ are adopted. These methods exhibit noticeable oscillations and irregular fluctuations in both training loss and relative error throughout the training process. 
In contrast, ODELoRA-RK4 consistently maintains a smooth decrease in both loss and error across all evaluated examples. 
Even in regimes where first-order methods suffer from severe instability, ODELoRA-RK4 continues to progress steadily toward lower error levels, demonstrating strong robustness to problem parameters.
These results further highlight the advantages of using higher-order discretization schemes for LoRA training when considering it as solving ODEs. By more faithfully approximating the underlying ODE dynamics, higher-order Runge--Kutta methods lead to improved stability and reliability, especially in stiff or sensitive regimes that commonly arise in physics-informed learning.

\section{Conclusion and Discussions}
\label{section6}

In this paper, we address the mismatch between LoRA fine-tuning and full fine-tuning by developing a continuous-time optimization flow for LoRA factors, termed \emph{ODELoRA}, which explicitly tracks the gradient flow of the full parameter matrix and satisfies the balanced manifold constraints.
In practice, we realize this flow by discretizing ODELoRA using well-established numerical solvers for ODEs, including Euler and Runge--Kutta methods.
This formulation provides a unified ODE-based perspective for understanding, analyzing, and designing optimization algorithms for LoRA fine-tuning.
We establish linear convergence for ODELoRA and its discretizations when considering strongly convex objectives, and further extend the analysis to the matrix sensing problem, which serves as a concrete and theoretically tractable example.
Moreover, motivated by the notion of \emph{stable feature learning}, a property known to be critical in deep learning, we show that ODELoRA achieves stable feature learning with a constant step size.
Notably, this step size does not scale with the model dimension, which is a highly desirable property for large-scale deep learning.
Finally, we validate our theoretical results through numerical experiments on matrix sensing, and  demonstrate the superior empirical performance of the proposed ODELoRA in fine-tuning physics-informed neural networks.

Several promising directions remain for future explorations.
First, although we adopt Runge--Kutta methods as reliable discretization schemes, these solvers are designed for broad classes of ODEs.
Given the specific structure of the ODELoRA flow, it would be of interest to develop customized numerical solvers that better exploit this structure and potentially lead to improved accuracy or efficiency.
Second, it is of interest to design optimization flows (ODEs) that explicitly enforce desirable structural properties of the LoRA factors and the full parameter matrix.
For example, more versatile manifold constraints for factors can be incorporated into the optimization flow to encode prior structural knowledge and to further stabilize and improve the training dynamics.

Overall, ODELoRA represents a first step toward an ODE-driven framework for LoRA fine-tuning.
By explicitly modeling the continuous-time dynamics and carefully designing discretization schemes, this perspective opens new possibilities for faithful optimization algorithm design for LoRA fine-tuning.

\newpage 
\bibliographystyle{plain}
\bibliography{main}

\newpage
\appendix

\section{Proof for \Cref{theorem5}}
\label{appendix1}

The optimization problem~\eqref{eq31}, which consists of a quadratic convex objective subject to linear equality constraints, admits a closed-form solution in general by vectorizing all matrix variables (via Kronecker products) and deriving the corresponding KKT conditions. The resulting solution can then be obtained by solving a linear system associated with the KKT matrix. However, although this approach is generally applicable to linearly constrained quadratic problems, the vectorization substantially enlarges the problem dimension and damage the matrix structures, leading to prohibitively high computational cost.

In this section, we exploit the special structure of the problem and show that the unique solution can instead be obtained through computationally efficient operations on small matrices. We proceed by analyzing the objective and the constraints separately.

We first consider the unconstrained problem associated with the quadratic objective:
\begin{equation*}
    \min_{\bJ(t),\bK(t)}\left\| \bK(t) \bA(t) + \bB(t)\bJ(t) - \left(- \nabla_{\bW} \mL(t)\right) \right\|_{\mathrm{F}}^{2}.
\end{equation*}
Since the objective is convex and quadratic, its global minimizer is characterized by the first-order stationarity conditions. The corresponding solution takes the form
\begin{equation*}
    \begin{split}
        & \bJ(t)  =  -\left(\bB(t)^{\top}\bB(t)\right)^{-1} \bB(t)^{\top} \nabla_{\bW} \mL(t) + \bX(t) \bA(t),\\
        & \bK(t)  = - \left[\bI - \bB(t) \left(\bB(t)^{\top}\bB(t)\right)^{-1}\bB(t)^{\top}\right] \nabla_{\bW} \mL(t) \bA(t)^{\top} \left(\bA(t) \bA(t)^{\top}\right)^{-1} - \bB(t) \bX(t),
    \end{split}
\end{equation*}
where $\bX(t)$ is an arbitrary matrix characterizing the non-uniqueness of the factorization in \eqref{eq9}.
To satisfy the equality constraints in~\eqref{eq31}, we now determine whether there exists a suitable choice of the auxiliary variable $\bX(t)$. Substituting the expressions for $\left(\bJ(t),\bK(t)\right)$ into the constraints leads to the following Sylvester equation:
\begin{equation*}
    \bH(t)\bX(t) + \bX(t)\bH(t) = \left(\bB(t)^{\top}\bB(t)\right)^{-1}\bB(t)^{\top} \nabla_{\bW} \mL(t)\bA(t)^{\top} + \bA(t)\nabla_{\bW} \mL(t)^{\top} \bB(t)\left(\bB(t)^{\top}\bB(t)\right)^{-1},
\end{equation*}
where $\bH(t):=\bA(t)\bA(t)^{\top} + \bB(t)^{\top}\bB(t)$.
If $\bH(t)$ is positive definite, then the above Sylvester equation admits a unique solution~\cite{higham2008functions}. 
With this choice of $\bX(t)$, the pair $\left(\bJ(t),\bK(t)\right)$ satisfies the equality constraints and simultaneously minimizes the unconstrained quadratic objective. Therefore, it constitutes the unique global solution of the original constrained optimization problem~\eqref{eq31}.

\section{Proof for Convergence of ODELoRA}

\subsection{Proof for \Cref{theorem1}}
\label{appendix1.1}
We consider the trajectory of $\bA(t)$ and $\bB(t)$ together as $\bW(t)$:

    \begin{equation}
    \label{eq25}
        \begin{split}
            & \frac{d}{dt}\left\| \bW(t) - \bW^{\star}\right\|_{\mathrm{F}}^2\\
            & = 2 \left\langle \frac{d \bW(t)}{dt}, \bW(t) - \bW^{\star}\right\rangle\\
        & = -2 \left\langle \nabla_{\bW} \mL(\bW(t)), \bW(t) - \bW^{\star}\right\rangle + 2\left\langle \bP_{\bB(t)}^{\text{null}}\nabla_{\bW} \mL(\bW(t))\bP_{\bA(t)}^{\text{null}}, \bW(t) - \bW^{\star}\right\rangle\\
        & \leq - 2(1-\varepsilon)  \left\langle \nabla_{\bW} \mL(\bW(t)), \bW(t) - \bW^{\star}\right\rangle \\
        & \leq -\mu(1-\varepsilon) \left\| \bW(t) - \bW^{\star}\right\|_{\mathrm{F}}^2,
        \end{split}
    \end{equation}
    where the last inequality makes use of the strong convexity of the objective that
    \begin{equation*}
        0 \geq \mL(\bW^{\star}) - \mL(\bW_t) \geq \left\langle \nabla_{\bW}\mL(\bW(t)), \bW^{\star} - \bW(t)\right\rangle + \frac{\mu}{2}\left\| \bW^{\star} - \bW(t)\right\|_{\mathrm{F}}^2.
    \end{equation*}
    Therefore, the linear convergence of ODELoRA is obtained from \eqref{eq25}:
    \begin{equation*}
        \left\| \bW(t) - \bW^{\star}\right\|_{\mathrm{F}}^2 \leq \exp\left(- \mu (1-\varepsilon) t\right) \left\| \bW(0)-\bW^{\star}\right\|_{\mathrm{F}}^2.
    \end{equation*}

\subsection{Proof for \Cref{theorem2}}
\label{appendix1.2}

Before presenting the detailed proof, we first establish two auxiliary lemmas that characterize the boundedness and Lipschitzness of the solution $\bX$ of the Sylvester equation and the induced field function $F$, given $(\bA,\bB)$ lying in the feasible set $\{(\bA,\bB): \gamma_{\min}\bI \preceq \bA\bA^{\top}\preceq\gamma_{\max}\bI~\text{and}~\gamma_{\min}\bI \preceq \bB^{\top}\bB\preceq\gamma_{\max}\bI\}$.

\begin{lemma}
\label{lemma1}
    Suppose that Condition 2 in \Cref{assumption1} and \Cref{assumption2} hold. Then, for any $(\bA,\bB)\in \mW$, the Sylvester equation
    \begin{equation*}
        \bH\bX + \bX \bH =\left(\bB^{\top}\bB\right)^{-1}\bB^{\top} \nabla\mL \bA^{\top} + \bA \nabla\mL^{\top}\bB\left(\bB^{\top}\bB\right)^{-1},
    \end{equation*}
    with $\bH = \bB^{\top}\bB + \bA\bA^{\top}$,
    admits a unique solution $\bX(\bA,\bB)$.
    Moreover, this solution is bounded, satisfying $\left\| \bX\right\|_{\mathrm{F}} \leq \frac{L \sqrt{\gamma_{\max}}}{2 \gamma_{\min}^{3/2}}$, and is Lipschitz continuous with respect to $(\bA,\bB)$ in $\mathcal{W}$, with the Lipschitz constant
    \begin{equation*}
        \frac{M \gamma_{\max}}{2\gamma_{\min}^{3/2}} + \frac{3L \gamma_{\max}^{3/2}}{\gamma_{\min}^3}.
    \end{equation*}
\end{lemma}

\begin{proof}
Since both $\bB^{\top}\bB$ and $\bA\bA^{\top}$ are positive definite for $(\bA,\bB) \in \mW$, the Sylvester equation admits a unique solution $\bX$; see, e.g.,~\cite{higham2008functions}.
Define $\bC = \left(\bB^{\top}\bB\right)^{-1}\bB^{\top} \nabla\mL \bA^{\top} + \bA \nabla\mL^{\top}\bB\left(\bB^{\top}\bB\right)^{-1}$. We first analyze the left-hand side of the equation. Under the positive definiteness condition, we have
    \begin{equation}
    \label{eq20}
        \left\| \bH\bX + \bX \bH\right\|_{\mathrm{F}} \geq 4\gamma_{\min} \left\| \bX\right\|_{\mathrm{F}}. 
    \end{equation}
    Since $\left\|\bC\right\|_{\mathrm{F}} \leq \frac{2L \sqrt{\gamma_{\max}}}{\sqrt{\gamma_{\min}}}$, we have that $\left\| \bX\right\|_{\mathrm{F}} \leq \frac{L \sqrt{\gamma_{\max}}}{2 \gamma_{\min}^{3/2}}$.

    Consequently, for any $(\Tilde{\bA},\Tilde{\bB}) \in \mW$ with the corresponding solution $\Tilde{\bX}$, the matrix $\Tilde{\bH}$, and the right-hand side function $\Tilde{\bC}$, we obtain
    \begin{equation*}
        \begin{split}
            \left\| \Tilde{\bX} - \bX\right\|_{\mathrm{F}} & \leq \frac{1}{4\gamma_{\min}} \left\| \bH\left(\Tilde{\bX} - \bX\right) + \left(\Tilde{\bX} - \bX\right) \bH\right\|_{\mathrm{F}} \\
            & \leq \frac{1}{4\gamma_{\min}} \left\| \bC - \Tilde{\bC}\right\|_{\mathrm{F}} + \frac{1}{4\gamma_{\min}} \left\| \left(\bH - \Tilde{\bH}\right)\Tilde{\bX} + \Tilde{\bX}\left(\bH - \Tilde{\bH}\right)\right\|_{\mathrm{F}}\\
            & \leq \frac{1}{4\gamma_{\min}} \left\| \bC - \Tilde{\bC}\right\|_{\mathrm{F}} + \frac{L\gamma_{\max}}{\gamma_{\min}^{5/2}} \left\| \left(\bA,\bB\right) - \left(\Tilde{\bA},\Tilde{\bB}\right)\right\|_{\mathrm{F}}
        \end{split}
    \end{equation*}
    Therefore, the Lipschitz continuity of $\bX$ with respect to $(\bA,\bB)$ follows from the Lipschitz continuity of $\bC$.

    We now bound $\|\bC-\Tilde{\bC}\|_{\mathrm{F}}$. For notational simplicity, we denote the loss function with $(\Tilde{\bA},\Tilde{\bB})$ as $\Tilde{\mL}$. Observe that
    \begin{equation*}
        \begin{split}
            \left\| \bC - \Tilde{\bC}\right\|_{\mathrm{F}} 
            & \leq 2 \left\| \left(\bB^{\top}\bB\right)^{-1}\bB^{\top} \nabla\mL \bA^{\top} - \left(\Tilde{\bB}^{\top}\Tilde{\bB}\right)^{-1}\Tilde{\bB}^{\top} \nabla\Tilde{\mL} \Tilde{\bA}^{\top}\right\|_{\mathrm{F}}\\
            & \leq 2\left\| \left(\bB^{\top}\bB\right)^{-1}\bB^{\top} \nabla\mL \bA^{\top} - \left(\bB^{\top}\bB\right)^{-1}\bB^{\top} \nabla\mL \Tilde{\bA}^{\top}\right\|_{\mathrm{F}}\\
            & \hspace{1em} + 2\left\| \left(\bB^{\top}\bB\right)^{-1}\bB^{\top} \nabla\mL \Tilde{\bA}^{\top} - \left(\bB^{\top}\bB\right)^{-1}\bB^{\top} \nabla\Tilde{\mL} \Tilde{\bA}^{\top}\right\|_{\mathrm{F}}\\
            & \hspace{1em} + 2\left\| \left(\bB^{\top}\bB\right)^{-1}\bB^{\top} \nabla\Tilde{\mL} \Tilde{\bA}^{\top} - \left(\Tilde{\bB}^{\top}\Tilde{\bB}\right)^{-1}\Tilde{\bB}^{\top} \nabla\Tilde{\mL} \Tilde{\bA}^{\top}\right\|_{\mathrm{F}}.
        \end{split}
    \end{equation*}
    Then, we bound these three terms separately. For the first term,
    \begin{equation*}
            \left\| \left(\bB^{\top}\bB\right)^{-1}\bB^{\top} \nabla\mL \bA^{\top} - \left(\bB^{\top}\bB\right)^{-1}\bB^{\top} \nabla\mL \Tilde{\bA}^{\top}\right\|_{\mathrm{F}}  \leq \frac{L}{\sqrt{\gamma_{\min}}} \left\| \bA - \Tilde{\bA}\right\|_{\mathrm{F}}.
    \end{equation*}
    For the second term, using the $M$-smoothness of $\mathcal{L}$,
    \begin{equation*}
        \begin{split}
            & \left\| \left(\bB^{\top}\bB\right)^{-1}\bB^{\top} \nabla\mL \Tilde{\bA}^{\top} - \left(\bB^{\top}\bB\right)^{-1}\bB^{\top} \nabla\Tilde{\mL} \Tilde{\bA}^{\top}\right\|_{\mathrm{F}}\\
            & \leq \frac{\sqrt{\gamma_{\max}}}{\sqrt{\gamma_{\min}}} \left\| \nabla \mL - \nabla \Tilde{\mL}\right\|_{\mathrm{F}} \leq \frac{M \sqrt{\gamma_{\max}}}{\sqrt{\gamma_{\min}}} \left\| \bB\bA - \Tilde{\bB}\Tilde{\bA}\right\|_{\mathrm{F}}\\
            & \leq  \frac{M  \gamma_{\max}}{\sqrt{\gamma_{\min}}} \left\| (\bA,\bB) - (\Tilde{\bA},\Tilde{\bB})\right\|_{\mathrm{F}}.
        \end{split}
    \end{equation*}
    For the third term, we have
    \begin{equation*}
        \begin{split}
            & \left\| \left(\bB^{\top}\bB\right)^{-1}\bB^{\top} \nabla\Tilde{\mL} \Tilde{\bA}^{\top} - \left(\Tilde{\bB}^{\top}\Tilde{\bB}\right)^{-1}\Tilde{\bB}^{\top} \nabla\Tilde{\mL} \Tilde{\bA}^{\top}\right\|_{\mathrm{F}} \\
            & \leq L \sqrt{\gamma_{\max}} \left\| \left(\bB^{\top}\bB\right)^{-1}\bB^{\top}  - \left(\Tilde{\bB}^{\top}\Tilde{\bB}\right)^{-1}\Tilde{\bB}^{\top} \right\|_{\mathrm{F}}\\
            & \leq L \sqrt{\gamma_{\max}} \left\| \left(\bB^{\top}\bB\right)^{-1}\bB^{\top}  - \left(\bB^{\top}\bB\right)^{-1}\Tilde{\bB}^{\top} \right\|_{\mathrm{F}}\\
            & \hspace{1em} + L \sqrt{\gamma_{\max}} \left\|\left(\bB^{\top}\bB\right)^{-1}\Tilde{\bB}^{\top} - \left(\Tilde{\bB}^{\top}\Tilde{\bB}\right)^{-1}\Tilde{\bB}^{\top}\right\|_{\mathrm{F}}\\
            & \leq \left(\frac{L \sqrt{\gamma_{\max}}}{\gamma_{\min}}  + \frac{2L \gamma_{\max}^{3/2}}{\gamma_{\min}^2}\right) \left\|\bB - \Tilde{\bB} \right\|_{\mathrm{F}}.
        \end{split}
    \end{equation*}
    Combining the above bounds, we conclude that 
    \begin{equation*}
        \begin{split}
            & \left\| \Tilde{\bX} - \bX\right\|_{\mathrm{F}} \\
            & \leq \left(\frac{L}{2\gamma_{\min}^{3/2}} + \frac{M \gamma_{\max}}{2\gamma_{\min}^{3/2}} + \frac{L \sqrt{\gamma_{\max}}}{2 \gamma_{\min}^2} + \frac{L \gamma_{\max}^{3/2}}{\gamma_{\min}^3} + \frac{L \gamma_{\max}}{\gamma_{\min}^{5/2}} \right) \left\| \left(\bA,\bB\right) - \left(\Tilde{\bA},\Tilde{\bB}\right)\right\|_{\mathrm{F}}\\
            & \leq \left( \frac{M \gamma_{\max}}{2\gamma_{\min}^{3/2}} + \frac{3L \gamma_{\max}^{3/2}}{\gamma_{\min}^3} \right) \left\| \left(\bA,\bB\right) - \left(\Tilde{\bA},\Tilde{\bB}\right)\right\|_{\mathrm{F}}.
        \end{split}
    \end{equation*}
\end{proof}

\begin{lemma}
\label{lemma2}
    The field function $F(\bA,\bB)$ is Lipschitz continuous with respect to $(\bA,\bB)$ on $\mW$, with Lipschitz constant
    \begin{equation*}
        \frac{18L \gamma_{\max}^{2}}{\gamma_{\min}^3} + \frac{3M \gamma_{\max}^{3/2}}{\gamma_{\min}^{3/2}}.
    \end{equation*}
\end{lemma}

\begin{proof}
From the explicit expression of $F$ in \eqref{eq13}, we observe that $F$ is composed of the Sylvester solution $\bX(\bA,\bB)$, the gradient $\nabla\mathcal{L}$, and matrix products and inverses involving $\bA$ and $\bB$.
Since all matrix operations appearing in $F$ are smooth on the feasible set $\mW$, the Lipschitz continuity of $F$ follows from the Lipschitz continuity of $\bX$ together with standard norm inequalities.
We explicitly construct the Lipschitz constant below.

First, by the $M$-smoothness of $\mL$, we have
    \begin{equation*}
        \begin{split}
            \left\| \nabla \mL - \nabla \Tilde{\mL}\right\|_{\mathrm{F}} \leq M \left\| \bB\bA - \Tilde{\bB}\Tilde{\bA}\right\|_{\mathrm{F}} \leq M\sqrt{\gamma_{\max}} \left\| (\bA,\bB) - (\Tilde{\bA},\Tilde{\bB})\right\|_{\mathrm{F}}.
        \end{split}
    \end{equation*}
    Next, consider the difference of the pseudoinverse terms:
    \begin{equation*}
        \begin{split}
            & \left\|\left(\bB^{\top}\bB\right)^{-1}\bB^{\top} - \left(\Tilde{\bB}^{\top}\Tilde{\bB}\right)^{-1}\Tilde{\bB}^{\top}\right\|_{\mathrm{F}}\\
            & \leq \left\| \left(\bB^{\top}\bB\right)^{-1}\bB^{\top}  - \left(\bB^{\top}\bB\right)^{-1}\Tilde{\bB}^{\top} \right\|_{\mathrm{F}} +  \left\|\left(\bB^{\top}\bB\right)^{-1}\Tilde{\bB}^{\top} - \left(\Tilde{\bB}^{\top}\Tilde{\bB}\right)^{-1}\Tilde{\bB}^{\top}\right\|_{\mathrm{F}}\\
            & \leq \left(\frac{1}{\gamma_{\min}} + \frac{2\gamma_{\max}}{\gamma_{\min}^2} \right) \left\|\bB - \Tilde{\bB} \right\|_{\mathrm{F}},
        \end{split}
    \end{equation*}
    and similarly,
    \begin{equation*}
        \left\| \bA^{\top} \left(\bA \bA^{\top}\right)^{-1} - \Tilde{\bA}^{\top} \left(\Tilde{\bA}\Tilde{\bA}^{\top}\right)^{-1} \right\|_{\mathrm{F}} \leq \left(\frac{1}{\gamma_{\min}} + \frac{2\gamma_{\max}}{\gamma_{\min}^2} \right) \left\|\bA - \Tilde{\bA} \right\|_{\mathrm{F}}.
    \end{equation*}
    Furthermore, for the projection term, we obtain
    \begin{equation*}
        \begin{split}
            & \left\| \bB \left(\bB^{\top}\bB\right)^{-1}\bB^{\top} - \Tilde{\bB} \left(\Tilde{\bB}^{\top}\Tilde{\bB}\right)^{-1}\Tilde{\bB}^{\top}\right\|_{\mathrm{F}} \\
            & \leq \left\| \bB \left(\bB^{\top}\bB\right)^{-1}\bB^{\top} - \Tilde{\bB} \left(\Tilde{\bB}^{\top}\Tilde{\bB}\right)^{-1}\bB^{\top}\right\|_{\mathrm{F}} \\
            & \hspace{1em} + \left\| \Tilde{\bB} \left(\Tilde{\bB}^{\top}\Tilde{\bB}\right)^{-1}\bB^{\top} - \Tilde{\bB} \left(\Tilde{\bB}^{\top}\Tilde{\bB}\right)^{-1}\Tilde{\bB}^{\top}\right\|_{\mathrm{F}} \\
            & \leq \left(\frac{\sqrt{\gamma_{\max}}}{\gamma_{\min}} + \frac{2\gamma_{\max}^{3/2}}{\gamma_{\min}^2} + \frac{1}{\sqrt{\gamma_{\min}}} \right) \left\|\bB - \Tilde{\bB} \right\|_{\mathrm{F}}.
        \end{split}
    \end{equation*}
Combining the above bounds, we obtain
\begin{equation*}
    \begin{split}
        & \left\| \left(\bB^{\top}\bB\right)^{-1} \bB^{\top} \nabla_{\bW} \mL - \left(\Tilde{\bB}^{\top}\Tilde{\bB}\right)^{-1} \Tilde{\bB}^{\top} \nabla_{\bW} \Tilde{\mL}\right\|_{\mathrm{F}} \\
        & \leq \left\| \left(\bB^{\top}\bB\right)^{-1} \bB^{\top} - \left(\Tilde{\bB}^{\top}\Tilde{\bB}\right)^{-1} \Tilde{\bB}^{\top}\right\|_{\mathrm{F}}  \left\| \nabla_{\bW} \mL\right\|_{\mathrm{F}} \\
        & \hspace{1em} + \left\| \left(\Tilde{\bB}^{\top}\Tilde{\bB}\right)^{-1} \Tilde{\bB}^{\top}\right\| \left\| \nabla_{\bW} \mL - \nabla_{\bW} \Tilde{\mL}\right\|_{\mathrm{F}} \\
        & \leq \left(\frac{L}{\gamma_{\min}} + \frac{2L\gamma_{\max}}{\gamma_{\min}^2}+ \frac{M \sqrt{\gamma_{\max}}}{\sqrt{\gamma_{\min}}} \right) \left\| (\bA,\bB) - (\Tilde{\bA},\Tilde{\bB})\right\|_{\mathrm{F}}\\
        & \leq \left(\frac{3L\gamma_{\max}}{\gamma_{\min}^2}+ \frac{M \sqrt{\gamma_{\max}}}{\sqrt{\gamma_{\min}}} \right) \left\| (\bA,\bB) - (\Tilde{\bA},\Tilde{\bB})\right\|_{\mathrm{F}},
    \end{split}
\end{equation*}
\begin{equation*}
    \begin{split}
        & \left\| \bX\bA - \Tilde{\bX}\Tilde{\bA}\right\|_{\mathrm{F}} \leq \left\| \bX - \Tilde{\bX} \right\|_{\mathrm{F}} \left\| \bA\right\| + \left\| \Tilde{\bX}\right\|_{\mathrm{F}} \left\|\bA - \Tilde{\bA} \right\|_{\mathrm{F}}\\
        & \leq \left(\frac{M \gamma_{\max}^{3/2}}{2\gamma_{\min}^{3/2}} + \frac{4L \gamma_{\max}^{2}}{\gamma_{\min}^3} \right) \left\| \left(\bA,\bB\right) - \left(\Tilde{\bA},\Tilde{\bB}\right)\right\|_{\mathrm{F}},
    \end{split}
\end{equation*}
and
\begin{equation*}
    \begin{split}
        \left\| \bB\bX - \Tilde{\bB}\Tilde{\bX}\right\|_{\mathrm{F}} \leq  \left(\frac{M \gamma_{\max}^{3/2}}{2\gamma_{\min}^{3/2}} + \frac{4L \gamma_{\max}^{2}}{\gamma_{\min}^3} \right) \left\| \left(\bA,\bB\right) - \left(\Tilde{\bA},\Tilde{\bB}\right)\right\|_{\mathrm{F}}.
    \end{split}
\end{equation*}

For notational simplicity, we define $\bP_{\bB}^{\text{null}} = \bI - \bB \left(\bB^{\top}\bB\right)^{-1}\bB^{\top}$ and $\bP_{\bA} = \bA^{\top} \left(\bA \bA^{\top}\right)^{-1}$, and analogously define $\bP_{\Tilde{\bB}}^{\text{null}}$ and $\bP_{\Tilde{\bA}}$. Then, by repeatedly applying the triangle inequality, we obtain
\begin{equation*}
    \begin{split}
        & \left\| \bP_{\bB}^{\text{null}} \nabla_{\bW} \mL \bP_{\bA} - \bP_{\Tilde{\bB}}^{\text{null}} \nabla_{\bW} \Tilde{\mL} \bP_{\Tilde{\bA}}\right\|_{\mathrm{F}}  \\
        & \leq \left\| \bP_{\bB}^{\text{null}} \nabla_{\bW} \mL \bP_{\bA} - \bP_{\Tilde{\bB}}^{\text{null}} \nabla_{\bW} \mL \bP_{\bA}\right\|_{\mathrm{F}} + \left\| \bP_{\Tilde{\bB}}^{\text{null}} \nabla_{\bW} \mL \bP_{\bA} - \bP_{\Tilde{\bB}}^{\text{null}} \nabla_{\bW} \Tilde{\mL} \bP_{\bA}\right\|_{\mathrm{F}}\\
        & \hspace{1em} + \left\| \bP_{\Tilde{\bB}}^{\text{null}} \nabla_{\bW} \Tilde{\mL} \bP_{\bA} - \bP_{\Tilde{\bB}}^{\text{null}} \nabla_{\bW} \Tilde{\mL} \bP_{\Tilde{\bA}}\right\|_{\mathrm{F}}\\
        & \leq \left(\frac{7L \gamma_{\max}^{3/2}}{\gamma_{\min}^{5/2}} + \frac{M \sqrt{\gamma_{\max}}}{\sqrt{\gamma_{\min}}}\right) \left\| \left(\bA,\bB\right) - \left(\Tilde{\bA},\Tilde{\bB}\right)\right\|_{\mathrm{F}}.
    \end{split}
\end{equation*}

    Based on the above inequalities and the Lipschitzness of $\bX$ established in \Cref{lemma1}, we conclude from the explicit expression of $F$ in \eqref{eq13} that
    \begin{equation*}
        \begin{split}
            & \left\|F(\bA,\bB) - F(\Tilde{\bA},\Tilde{\bB}) \right\|_{\mathrm{F}} \\
            & \leq \left(\frac{18L \gamma_{\max}^{2}}{\gamma_{\min}^3} +  \frac{2M \sqrt{\gamma_{\max}}}{\sqrt{\gamma_{\min}}} + \frac{M \gamma_{\max}^{3/2}}{\gamma_{\min}^{3/2}}\right) \left\| \left(\bA,\bB\right) - \left(\Tilde{\bA},\Tilde{\bB}\right)\right\|_{\mathrm{F}}\\
            & \leq \left(\frac{18L \gamma_{\max}^{2}}{\gamma_{\min}^3} + \frac{3M \gamma_{\max}^{3/2}}{\gamma_{\min}^{3/2}}\right) \left\| \left(\bA,\bB\right) - \left(\Tilde{\bA},\Tilde{\bB}\right)\right\|_{\mathrm{F}}.
        \end{split}
    \end{equation*}

    Moreover, we derive the boundedness of $F$. We notice from \eqref{eq13} that
    \begin{equation*}
        \begin{split}
            & \left\| F(\bA,\bB)\right\|_{\mathrm{F}} \\
            & \leq \left\| \left(\bB^{\top}\bB\right)^{-1} \bB^{\top} \nabla_{\bW} \mL\right\|_{\mathrm{F}} + \left\|\bX\bA \right\|_{\mathrm{F}} + \left\| \bP_{\bB}^{\text{null}} \nabla_{\bW} \mL \bP_{\bA} \right\|_{\mathrm{F}} + \left\| \bB\bX\right\|_{\mathrm{F}} \\
            & \leq \frac{2L}{\sqrt{\gamma_{\min}}} + \frac{L \gamma_{\max}}{ \gamma_{\min}^{3/2}}.
        \end{split}
    \end{equation*}
\end{proof}

We are now ready to establish the convergence of the discretized ODELoRA scheme for problem~\eqref{eq_strongly_convex_lora}.
To avoid redundancy and repeated arguments, we focus on ODELoRA-RK4, as the Euler and RK2 methods can be treated as simpler special cases.
For notational convenience, we define $\bZ := (\bA,\bB)$.
Accordingly, we write $\bZ_t := (\bA_t,\bB_t)$ and denote the intermediate RK4 stages by
$\bZ_t^{(i)} := (\bA_t^{(i)},\bB_t^{(i)})$ for $i \in \{1,2,3\}$.

We observe from \eqref{eq20} that 
\begin{equation*}
    \left\| \bX_t \right\|_{\mathrm{F}} \leq \frac{1}{2\gamma_{\min}}\left\| \left(\bB_{t}^{\top}\bB_{t}\right)^{-1}\bB_{t}^{\top} \nabla_{\bW}\mL_{t} \bA_{t}^{\top} \right\|_{\mathrm{F}} \leq \frac{\sqrt{\gamma_{\max}}}{2\gamma_{\min}^{3/2}} \left\|\nabla_{\bW} \mL_t \right\|_{\mathrm{F}}.
\end{equation*}
Consequently, it follows from \eqref{eq13} that
\begin{equation*}
    \left\|F(\bZ_{t}) \right\|_{\mathrm{F}} \leq \left(\frac{2}{\sqrt{\gamma_{\min}}} + \frac{\gamma_{\max}}{\gamma_{\min}^{3/2}}\right) \left\|\nabla_{\bW} \mL_t \right\|_{\mathrm{F}}.
\end{equation*}

From the update scheme of RK4 in \eqref{eq16} and the Lipschitz continuity of $F$ established in \Cref{lemma2}, we obtain
\begin{equation*}
    \begin{split}
        \left\|F(\bZ_{t}^{(1)})- F(\bZ_{t}) \right\|_{\mathrm{F}} & \leq C_{1,1} \left\| \bZ_{t}^{(1)} - \bZ_{t}\right\|_{\mathrm{F}} = \frac{C_{1,1} h}{2} \left\|F(\bZ_t) \right\|_{\mathrm{F}}\\
        & \leq \frac{C_{1,1} h}{2} \left(\frac{2}{\sqrt{\gamma_{\min}}} + \frac{\gamma_{\max}}{\gamma_{\min}^{3/2}}\right) \left\|\nabla_{\bW}\mL_{t} \right\|_{\mathrm{F}} \\
        & := C_{1,2} h \left\|\nabla_{\bW}\mL_{t} \right\|_{\mathrm{F}}
    \end{split}
\end{equation*}
where $C_{1,1}$ denotes the Lipschitz constant of $F$, and $C_{1,2}$ is the constant depending only on $M$, $L$, $\gamma_{\max}$, and $\gamma_{\min}$. 
Here, $C_{1,1}$ and $C_{1,2}$ depends linearly on $M$, since the Lipschitz constant in \Cref{lemma2}.
These constants are independent of the iteration index $t$ and the step size $h$, and depend only on the quantities specified in Assumptions~\ref{assumption1} and~\ref{assumption2}.
Similarly, we have
\begin{equation*}
    \left\|F(\bZ_{t}^{(2)})- F(\bZ_{t}) \right\|_{\mathrm{F}} \leq C_{1,1} \left\| \bZ_{t}^{(2)} - \bZ_{t}\right\|_{\mathrm{F}} = \frac{C_{1,1} h}{2} \left\|F(\bZ_{t}^{(1)}) \right\|_{\mathrm{F}} \leq C_{1,3} h \left\|\nabla_{\bW}\mL_{t} \right\|_{\mathrm{F}},
\end{equation*}
and
\begin{equation*}
    \left\|F(\bZ_{t}^{(3)})- F(\bZ_{t}) \right\|_{\mathrm{F}} \leq C_{1,1} \left\| \bZ_{t}^{(3)} - \bZ_{t}\right\|_{\mathrm{F}} = C_{1,1} h \left\|F(\bZ_{t}^{(2)}) \right\|_{\mathrm{F}} \leq C_{1,4} h \left\|\nabla_{\bW}\mL_{t} \right\|_{\mathrm{F}},
\end{equation*}
where $C_{1,3}, C_{1,4}$ are constants depending only on $M$, $L$, $\gamma_{\max}$, and $\gamma_{\min}$, and scale linearly on $M$ as $C_{1,2}$.
Therefore, 
\begin{equation}
        \left\|\frac{1}{6}\left(F(\bZ_t) + 2F(\bZ_{t}^{(1)}) + 2F(\bZ_{t}^{(2)}) + F(\bZ_{t}^{(3)})\right)  - F(\bZ_t)\right\|_{\mathrm{F}}  \leq C_{1,4} h \left\|\nabla_{\bW}\mL_{t} \right\|_{\mathrm{F}}.
\end{equation}
As a result,
\begin{equation*}
    \left\| (\Delta\bA_t, \Delta\bB_t) - h F(\bZ_t) \right\|_{\mathrm{F}} \leq C_{1,4}h^2 \left\|\nabla_{\bW}\mL_{t} \right\|_{\mathrm{F}},
\end{equation*}
as $(\Delta\bA_t, \Delta\bB_t) = \frac{h}{6}\left(F(\bZ_t) + 2F(\bZ_{t}^{(1)}) + 2F(\bZ_{t}^{(2)}) + F(\bZ_{t}^{(3)})\right)$.

Finally, we observe that
\begin{equation*}
    \begin{split}
    \bW_{t+1} - \bW_{t} & = \bB_{t+1}\bA_{t+1} - \bB_t \bA_t \\
    & = \bB_t \Delta\bA_t + \Delta\bB_t \bA_t + \Delta\bB_t\Delta\bA_t\\
    & = -h \nabla_{\bW} \mL_{t} + h \bP_{\bB_t}^{\text{null}}\nabla_{\bW} \mL_{t}\bP_{\bA_t}^{\text{null}} + \bE,
    \end{split}
\end{equation*}
where the residual term $\bE$ satisfies $\|\bE\|_{\mathrm{F}} \leq C_{1,5}h^2\left\|\nabla_{\bW}\mL_{t} \right\|_{\mathrm{F}}$, for a constant $C_{1,5}$ depending only on $M$, $L$, $\gamma_{\max}$, and $\gamma_{\min}$, and scales linearly with $M$.
Therefore, by the smoothness of $\mL$, we have
\begin{equation*}
    \begin{split}
        \mL(\bW_{t+1}) & \leq \mL(\bW_{t}) + \left\langle \nabla_{\bW} \mL(\bW_t), \bW_{t+1} - \bW_{t} \right\rangle + \frac{M}{2} \left\|\bW_{t+1} - \bW_{t}  \right\|_{\mathrm{F}}^{2}\\
        & \leq \mL(\bW_{t}) - h \left\|\nabla_{\bW}\mL_{t} \right\|_{\mathrm{F}}^2  + h \left\langle \bP_{\bB_t}^{\text{null}}\nabla_{\bW} \mL_{t}\bP_{\bA_t}^{\text{null}}, \nabla_{\bW} \mL_{t}\right\rangle + C_{1,5} h^2 \left\|\nabla_{\bW}\mL_{t} \right\|_{\mathrm{F}}^2 \\
        & \hspace{1em} + \frac{M}{2}\left(h\left\| \nabla_{\bW}\mL_t\right\|_{\mathrm{F}} + \left\| \bE\right\|_{\mathrm{F}}\right)^2 \\
        & \leq \mL(\bW_{t}) - h (1-\varepsilon) \left\|\nabla_{\bW} \mL_{t}\right\|_{\mathrm{F}}^{2} + C_{1,6}h^2 \left\|\nabla_{\bW}\mL_{t} \right\|_{\mathrm{F}}^{2},
    \end{split}
\end{equation*}
where the constant $C_{1,6}$ depends linearly on the smoothness parameter $M$, and polynomially on $L$, $\gamma_{\max}$, and $\gamma_{\min}$. Choosing the step size such that $h \leq \frac{1-\varepsilon}{2C_{1,6}} = \mathcal{O}(\frac{1}{M})$, we obtain
\begin{equation*}
    \mL(\bW_{t+1}) \leq \mL(\bW_{t}) - \frac{(1-\varepsilon)h}{2} \left\|\nabla_{\bW}\mL_{t} \right\|_{\mathrm{F}}^{2}.
\end{equation*}
By the PL inequality for strongly convex function that
\begin{equation*}
    \left\|\nabla_{\bW}\mL_{t} \right\|_{\mathrm{F}}^{2} \geq 2 \mu \left(\mL(\bW_{t}) - \mL(\bW^{\star})\right),
\end{equation*}
we have 
\begin{equation*}
    \mL(\bW_{t+1}) - \mL(\bW^{\star}) \leq \left(1 - (1-\epsilon)\mu h\right) \left(\mL(\bW_{t}) - \mL(\bW^{\star})\right),
\end{equation*}
which completes the proof.

\subsection{Proof for \Cref{theorem3}}

Before presenting the detailed proof, we first establish the following useful lemma.
\begin{lemma}
\label{lemma3}
The following bounds hold:
    \begin{equation*}
        \left\|  \bA^{\star}\bP_{\bA_t}^{\text{null}} \right\|_{\mathrm{F}} \leq \frac{1}{\sigma_{\min}(\bB^{\star})} \left\| \bB_t\bA_t - \bB^{\star}\bA^{\star} \right\|_{\mathrm{F}},
    \end{equation*}
    and
    \begin{equation*}
        \left\| \bP_{\bB_t}^{\text{null}} \bB^{\star} \right\|_{\mathrm{F}} \leq \frac{1}{\sigma_{\min}(\bA^{\star})} \left\| \bB_t\bA_t - \bB^{\star}\bA^{\star} \right\|_{\mathrm{F}}.
    \end{equation*}
\end{lemma}

\begin{proof}
    Observe that 
    \begin{equation*}
        \bP_{\bB_t}^{\text{null}} \bB^{\star} = \bP_{\bB_t}^{\text{null}}\left(\bB^{\star} + \bB_t \bQ\right),
    \end{equation*}
    for any matrix $\bQ$. Therefore, 
    \begin{equation*}
       \left\| \bP_{\bB_t}^{\text{null}} \bB^{\star} \right\|_{\mathrm{F}} \leq \min_{\bQ} \left\|\bP_{\bB_t}^{\text{null}}\left(\bB^{\star} + \bB_t \bQ\right) \right\|_{\mathrm{F}}. 
    \end{equation*}
    Let $\bE:=\bB^{\star}\bA^{\star}-\bB_t\bA_t$. Multiplying both two sides on the right by $\bA^{\star \top}$ leads to
    \begin{equation*}
        \bB^{\star}\bA^{\star}\bA^{\star \top} = \bB_t\bA_t\bA^{\star \top} + \bE\bA^{\star \top},
    \end{equation*}
    and hence
    \begin{equation*}
        \bB^{\star} = \bB_t\bA_t\bA^{\star \top}\left(\bA^{\star}\bA^{\star \top}\right)^{-1} + \bE\bA^{\star \top}\left(\bA^{\star}\bA^{\star \top}\right)^{-1}.
    \end{equation*}
    Let $\bQ = -\bA_t\bA^{\star \top}\left(\bA^{\star}\bA^{\star \top}\right)^{-1}$, we obtain 
    \begin{equation*}
        \bB^{\star} + \bB_t \bQ = \bE\bA^{\star \top}\left(\bA^{\star}\bA^{\star \top}\right)^{-1}.
    \end{equation*}
    Therefore,
    \begin{equation*}
       \begin{split}
           \left\| \bP_{\bB_t}^{\text{null}} \bB^{\star} \right\|_{\mathrm{F}} & \leq \min_{\bQ} \left\|\bP_{\bB_t}^{\text{null}}\left(\bB^{\star} + \bB_t \bQ\right) \right\|_{\mathrm{F}} \leq  \left\|\bE\right\|_{\mathrm{F}} \cdot \left\|\bA^{\star \top}\left(\bA^{\star}\bA^{\star \top}\right)^{-1}\right\|\\
           & \leq \frac{1}{\sigma_{\min}(\bA^{\star})} \left\| \bB_t\bA_t - \bB^{\star}\bA^{\star} \right\|_{\mathrm{F}}.
       \end{split}
    \end{equation*}
    A similar derivation holds that
    \begin{equation*}
       \left\|  \bA^{\star} \bP_{\bA_t}^{\text{null}}\right\|_{\mathrm{F}} \leq  \frac{1}{\sigma_{\min}(\bB^{\star})} \left\| \bB_t\bA_t - \bB^{\star}\bA^{\star} \right\|_{\mathrm{F}}.
    \end{equation*}

\end{proof}

The rank-$2r$ RIP of the sensing matrix $\bS$, together with \eqref{eq26}, implies that \Cref{assumption1} holds on the set $\mW := \{\bW_{\text{pt}}+\bB\bA: \bA \in \mathbb{R}^{r \times n}, \bB \in \mathbb{R}^{m \times r}\}$. Indeed, for any $\bW = \bW_{\text{pt}}+\bB\bA \in \mW$ and $\Tilde{\bW} = \bW_{\text{pt}}+\Tilde{\bB}\Tilde{\bA} \in \mW$, and $\alpha \in [0,1]$, we have $\alpha \bW + (1-\alpha)\Tilde{\bW} = \bW_{\text{pt}}+ \Delta\bW$, with rank of $\Delta\bW$ less than $2r$.
Therefore, the rank-$2r$ RIP controls the geometry of the linear map $\bW \mapsto \bW\,\bS$ along the entire line segment between $\bW$ and $\tilde{\bW}$, which establishes \Cref{assumption1} on $\mathcal{W}$.

We first note that 
    \begin{equation*}
    \bP_{\bB_t}^{\text{null}}\nabla_{\bW} \mL_{t}\bP_{\bA_t}^{\text{null}} = \bP_{\bB_t}^{\text{null}}\left(\bB_t\bA_t - \bB^{\star}\bA^{\star}\right) \bS \bS^{\top}\bP_{\bA_t}^{\text{null}} = - \bP_{\bB_t}^{\text{null}} \bB^{\star}\bA^{\star} \bS \bS^{\top}\bP_{\bA_t}^{\text{null}}.
    \end{equation*}
    Next, we decompose $\bS\bS^\top = (\bS\bS^\top-\bI)+\bI$ and apply \cite{candes2011tight} and \Cref{lemma3}:
    \begin{equation*}
        \begin{split}
            & \left\| \bP_{\bB_t}^{\text{null}}\nabla_{\bW} \mL_{t}\bP_{\bA_t}^{\text{null}}  \right\|_{\mathrm{F}}\\
            & \leq  \left\| \bP_{\bB_t}^{\text{null}} \bB^{\star}\bA^{\star} \left(\bS \bS^{\top} - \bI\right)\bP_{\bA_t}^{\text{null}} \right\|_{\mathrm{F}} + \left\| \bP_{\bB_t}^{\text{null}} \bB^{\star}\bA^{\star} \bP_{\bA_t}^{\text{null}} \right\|_{\mathrm{F}} \\
            & \leq \left\| \bP_{\bB_t}^{\text{null}} \bB^{\star}\bA^{\star} \left(\bS \bS^{\top} - \bI\right) \right\|_{\mathrm{F}} + \frac{1}{\sigma_{\min}(\bA^{\star}) \sigma_{\min}(\bB^{\star})} \left\| \bB_t\bA_t - \bB^{\star}\bA^{\star}\right\|_{\mathrm{F}}^{2}\\
            & \leq \delta \left\| \bP_{\bB_t}^{\text{null}} \bB^{\star}\bA^{\star} \right\|_{\mathrm{F}} + \frac{1}{\sigma_{\min}(\bA^{\star}) \sigma_{\min}(\bB^{\star})}  \left\| \bB_t\bA_t - \bB^{\star}\bA^{\star}\right\|_{\mathrm{F}}^{2}\\
            & \leq \delta \frac{\sigma_{\max}(\bA^{\star})}{\sigma_{\min}(\bA^{\star})}  \left\| \bB_t\bA_t - \bB^{\star}\bA^{\star}\right\|_{\mathrm{F}} + \frac{1}{\sigma_{\min}(\bA^{\star}) \sigma_{\min}(\bB^{\star})}  \left\| \bB_t\bA_t - \bB^{\star}\bA^{\star}\right\|_{\mathrm{F}}^{2}.
        \end{split}
    \end{equation*}

We then prove by induction. Suppose that
\begin{equation*}
    \frac{\delta \frac{\sigma_{\max}(\bA^{\star})}{\sigma_{\min}(\bA^{\star})}  +  \frac{1}{\sqrt{1-\delta} \cdot \sigma_{\min}(\bA^{\star}) \sigma_{\min}(\bB^{\star})} \left\| \left(\bB_t\bA_t - \bB^{\star}\bA^{\star}\right)\bS\right\|_{\mathrm{F}}}{1-\delta} \leq \epsilon.
\end{equation*}
Then,
    \begin{equation*}
        \begin{split}
            & \left\| \bP_{\bB_t}^{\text{null}}\nabla_{\bW} \mL_{t}\bP_{\bA_t}^{\text{null}}  \right\|_{\mathrm{F}}\\
            & \leq \left(\delta \frac{\sigma_{\max}(\bA^{\star})}{\sigma_{\min}(\bA^{\star})}  + \frac{1}{\sigma_{\min}(\bA^{\star}) \sigma_{\min}(\bB^{\star})}  \left\| \bB_t\bA_t - \bB^{\star}\bA^{\star}\right\|_{\mathrm{F}} \right) \left\| \bB_t\bA_t - \bB^{\star}\bA^{\star}\right\|_{\mathrm{F}}\\
            & \leq \left(\delta \frac{\sigma_{\max}(\bA^{\star})}{\sigma_{\min}(\bA^{\star})}  +  \frac{1}{\sqrt{1-\delta} \cdot \sigma_{\min}(\bA^{\star}) \sigma_{\min}(\bB^{\star})} \left\| \left(\bB_t\bA_t - \bB^{\star}\bA^{\star}\right)\bS\right\|_{\mathrm{F}}\right) \left\| \bB_t\bA_t - \bB^{\star}\bA^{\star}\right\|_{\mathrm{F}}\\
            & \leq (1-\delta) \varepsilon \left\| \bB_t\bA_t - \bB^{\star}\bA^{\star}\right\|_{\mathrm{F}}\\
            & \leq \varepsilon \left\| \nabla_{\bW} \mL_{t}\right\|_{F},
        \end{split}
    \end{equation*}
    i.e., the condition in \eqref{eq19} holds.

    Moreover, \Cref{theorem2} ensures that $\{\mathcal{L}(\bW_t)\}_{t\ge 0}$ is non-increasing. Since
$\mathcal{L}(\bW_t)=\tfrac12\|(\bB_t\bA_t-\bB^\star\bA^\star)\bS\|_{\mathrm{F}}^2$ in the matrix sensing setting,
it follows that $\|(\bB_t\bA_t-\bB^\star\bA^\star)\bS\|_{\mathrm{F}}$ is also non-increasing, and thus \eqref{eq27} propagates from $t=0$ to all $t\in\mathbb{Z}_+$.
Therefore, Theorem~\ref{theorem2} applies and yields linear convergence for ODELoRA with Euler, RK2, and RK4.


\section{Proof for Stable Feature Learning (\Cref{theorem4})}
\label{appendix2}

Recall that under the considered setting, the step size satisfies $h=\Theta(1)$, $\left\|\bs\right\|=\Theta(1)$, and $\left\|\bW_{\text{pt}}\bs - \by \right\| = \Theta(1)$.
We prove by induction.
Assume that at iteration $t$, the following scaling relations hold:
\begin{equation}
\label{eq29}
    \begin{split}
        & \left\|\bB_t(:,j)\right\| = 0,\quad\left\|\bA_t(j,:)\right\| =0,\quad \left\|\bA_t \bs\right\| = \Theta(1),\quad \left\|\bW_{t}\bs - \by \right\| = \Theta(1),\\
        & \left\|(\bA_t\bA_t^{\top})^{-1}\right\| = \Theta(1),\quad \left\|\left(\bB_t\bB_t^{\top}\right)^{-1}\right\|=\Theta(1).
    \end{split}
\end{equation}

We first analyze the term $\varphi_{t}^{(1)}$. By definition,
\begin{equation*}
    \begin{split}
        \varphi_{t}^{(1)} & = \frac{h}{6} F_{\bB}\left(\bA_t,\bB_t\right)\bA_t \bs \\
     & = \frac{h}{6} \left(- \bP_{\bB_t}^{\text{null}} \nabla_{\bB}\mL_t\left(\bA_t \bA_t^{\top}\right)^{-1} - \bB_t\bX_t \right) \bA_t\bs.
    \end{split}
\end{equation*}
We note that 
\begin{equation*}
\nabla_{\bB}\mL_t\left(\bA_t \bA_t^{\top}\right)^{-1} \bA_t\bs  = \left(\bW_{t}\bs - \by\right) \bs^{\top} \bA_t^{\top} \left(\bA_t \bA_t^{\top}\right)^{-1} \bA_t\bs.
\end{equation*}
Therefore, we have
\begin{equation*}
    \left\| \bP_{\bB_t}^{\text{null}}\nabla_{\bB}\mL_t\left(\bA_t \bA_t^{\top}\right)^{-1} \bA_t\bs\right\| = \Theta(1).
\end{equation*}
For the Sylvester equation
$\bH_t\bX_t + \bX_t \bH_t = \left(\bB_t^{\top}\bB_t\right)^{-1}\nabla_{\bA}\mL_t \bA_t^{\top} + \bA_t\nabla_{\bA}\mL_t^{\top}\left(\bB_t^{\top}\bB_t\right)^{-1}$,
we observe that 
\begin{equation*}
    \left(\bB_t^{\top}\bB_t\right)^{-1}\nabla_{\bA}\mL_t \bA_t^{\top} = \left(\bB_t^{\top}\bB_t\right)^{-1} \bB_t^{\top} \left(\bW_{t}\bs - \by\right) \bs^{\top} \bA_t^{\top},
\end{equation*}
which implies that $\left\| \left(\bB_t^{\top}\bB_t\right)^{-1}\nabla_{\bA}\mL_t \bA_t^{\top}\right\| = \Theta(1)$, and thus $\left\| \bX_t\right\| = \Theta(1)$ and $\left\| \bB_t\bX_t\bA_t\bs\right\| = \Theta(1)$. These jointly show that $\left\| \varphi_{t}^{(1)}\right\| = \Theta(1)$, since $h = \Theta(1)$.

For $\varphi_{t}^{(2)}$, we have
\begin{equation*}
    \begin{split}
        \varphi_{t}^{(2)} & = -\frac{h}{6} \bB_t\left(\bB_t^{\top}\bB_t\right)^{-1} \nabla_{\bA}\mL_t\bs + \frac{h}{6} \bB_t\bX_t \bA_t\bs \\
        & = -\frac{h}{6} \bB_t\left(\bB_t^{\top}\bB_t\right)^{-1} \bB_t^{\top}\left(\bW_{t}\bs - \by\right) \bs^{\top}  \bs + \frac{h}{6} \bB_t\bX_t \bA_t\bs.
    \end{split}
\end{equation*}
Using the induction hypotheses, we obtain that 
\begin{equation*}
    \left\| \bB_t\left(\bB_t^{\top}\bB_t\right)^{-1} \bB_t^{\top}\left(\bW_{t}\bs - \by\right) \bs^{\top}  \bs\right\| = \Theta(1).
\end{equation*} 
Therefore, $\left\|\varphi_{t}^{(2)}\right\|=\Theta(1)$.

Next, we examine the first intermediate RK4 stage. By the update scheme,
\begin{equation*}
    \begin{split}
        \bA_{t}^{(1)} & = \bA_{t} + \frac{h}{2}\left(-\left(\bB_t^{\top}\bB_t\right)^{-1}  \nabla_{\bA}\mL_t + \bX_t \bA_t,\right)\\
        & =  \bA_{t} + \frac{h}{2}\left(-\left(\bB_t^{\top}\bB_t\right)^{-1} \bB_t^{\top} \left(\bW_{t}\bs - \by\right) \bs^{\top}  + \bX_t \bA_t,\right),\\
        \bB_{t}^{(1)} & = \bB_{t} + \frac{h}{2}\left(- \bP_{\bB_t}^{\text{null}} \frac{\nabla_{\bB}\mL_t }{\sqrt{\bV_t}}\left(\bA_t \bA_t^{\top}\right)^{-1} - \bB_t\bX_t\right)\\
        & = \bB_{t} + \frac{h}{2}\left(- \bP_{\bB_t}^{\text{null}} \left(\bW_{t}\bs - \by\right) \bs^{\top}\bA_t^{\top}\left(\bA_t \bA_t^{\top}\right)^{-1} - \bB_t\bX_t\right)
    \end{split}
\end{equation*}
Since $\left\| \bX_t\right\| = \Theta(1)$ and $j$-th row of $\bA_t$ and column of $\bB_t$ satisfy $\left\|\bB_t(:,j)\right\| = 0, \left\|\bA_t(j,:)\right\| =0$, we have that $j$-th row of $\bX_t\bA_t$ and column of $\bB_t\bX_t$ are of $\Theta(1)$ in Euclidean norm. 
We also note that $\left\| \left(\bB_t^{\top}\bB_t\right)^{-1} \bB_t^{\top} \left(\bW_{t}\bs - \by\right) \right\| = \Theta(1)$ and $\left\|\bs^{\top}\bA_t^{\top}\left(\bA_t \bA_t^{\top}\right)^{-1} \right\| = \Theta(1)$.
Therefore, $j$-th row and column of $\bA_{t}^{(1)} - \bA_{t}$ and $\bB_{t}^{(1)} - \bB_{t}$ are of $\Theta(1)$ in Euclidean norm. 
Hence, we have that $\left\|\bA_{t}^{(1)}(j,:)\right\| = \Theta(1)$ and $\left\|\bB_{t}^{(1)}(:,j)\right\| = \Theta(1)$.
This further implies that $\left\|(\bA_t^{(1)}\bA_t^{(1)\top})^{-1}\right\| = \Theta(1),  \left\|\left(\bB_t^{(1)}\bB_t^{(1)\top}\right)^{-1}\right\|=\Theta(1)$.
We also find that 
\begin{equation*}
    \left\|\left(\bB_t^{\top}\bB_t\right)^{-1} \bB_t^{\top} \left(\bW_{t}\bs - \by\right) \bs^{\top} \bs\right\| = \Theta(1),\quad \left\| \bX_t\bA_t\bs\right\| = \Theta(1).
\end{equation*}
Therefore, $\left\|\bA_t^{(1)} \bs\right\| = \Theta(1)$. 
Moreover,
\begin{equation*}
    \bW_t^{(1)}\bs - \by = \bW_0  \bs - \by + \bB_t^{(1)} \bA_t^{(1)}\bs,
\end{equation*}
where $\left\| \bB_t^{(1)} \bA_t^{(1)}\bs \right\| = \Theta(1)$. Hence, we have $\left\|\bW_t^{(1)}\bs - \by  \right\| = \Theta(1)$. 
We conclude that all scaling properties in \eqref{eq29} assumed at $\bA_t$ and $\bB_t$ also hold for $\bA_t^{(1)}$ and $\bB_t^{(1)}$.

For $\varphi_{t}^{(3)}$, we observe that 
\begin{equation*}
    \begin{split}
        \varphi_{t}^{(3)} & = \frac{h}{3} F_{\bB}\left(\bA_t^{(1)},\bB_t^{(1)}\right)\bA_t \bs\\
        & = \frac{h}{3} \left(- \bP_{\bB_t^{(1)}}^{\text{null}} \nabla_{\bB}\mL_t^{(1)} \left(\bA_t^{(1)} \bA_t^{(1)\top}\right)^{-1} - \bB_t^{(1)}\bX_t^{(1)} \right) \bA_t^{(1)}\bs.
    \end{split}
\end{equation*}
We conduct similar analysis as $\varphi_{t}^{(1)}$ for $\varphi_{t}^{(3)}$ and conclude that $\left\|\varphi_{t}^{(3)}\right\| = \Theta(1)$. 
For $\varphi_t^{(4)}$, the expression is completely analogous to that of $\varphi_t^{(2)}$ with $(\bA_t,\bB_t,\bX_t)$ replaced by $(\bA_t^{(1)},\bB_t^{(1)},\bX_t^{(1)})$. Therefore, by the same argument, we obtain $\left\|\varphi_{t}^{(4)}\right\| = \Theta(1)$.

We can similarly show that the properties \eqref{eq29} for $(\bA_t,\bB_t)$ also holds for $(\bA_t^{(2)},\bB_t^{(2)})$ and $(\bA_t^{(3)},\bB_t^{(3)})$.
Since the remaining update terms $\{\varphi_t^{(i)}\}_{i=5}^8$ have the same algebraic structure and involve quantities with identical scaling behavior, the same reasoning applies. Consequently, we conclude that $\left\|\varphi_t^{(i)}\right\|=\Theta(1)$, for $i \in [8]$. Here, we omit all the redundant and repeated derivations that follow $\varphi_t^{(1)}$ and $\varphi_t^{(2)}$.

It remains to verify that conditions in \eqref{eq29} are satisfied at $t=1$. 
At initialization, the entries of $\bA_0$ are sampled with $\left\|\bA_0(j,:)\right\|=\Theta(1)$ and $\left\|\bA_0(j,:)\bs\right\| = \Theta(1)$, we have $\left\|\bA_0\bs\right\| = \Theta(1)$ and $\left\|\bA_0\bA_0^{\top}\right\|=\Theta(1)$,
under the condition $\|\bs\|=\Theta(1)$.
The matrix $\bB_0$ is initialized as a zero matrix, thus we add an additional regularizer for the singular matrix $\bB_0^{\top}\bB_0$, when calculating the inverse, i.e., $\left(\bB_0^{\top}\bB_0 + \epsilon \bI\right)^{-1}$ with $\epsilon >0$ irrelevant to dimension $n$. The first intermediate stage $\bA_0^{(1)}$ admits 
\begin{equation*}
    \begin{split}
        \bA_{0}^{(1)} & = \bA_{0} + \frac{h}{2} \left(-\left(\bB_0^{\top}\bB_0 + \epsilon\bI\right)^{-1} \nabla_{\bA}\mL_0+ \bX_0 \bA_0\right)\\
        & = \bA_{0} + \frac{h}{2} \left(-\left(\bB_0^{\top}\bB_0 + \epsilon\bI\right)^{-1} \bB_0 \nabla_{\bW}\mL_0+ \bX_0 \bA_0\right),
    \end{split}
\end{equation*}
Since $\bB_0 = \bm{0}$, we have $\nabla_{\bA}\mL_0 = 0$, and the Sylvester equation yields $\bX_0 = \bm{0}$. Therefore, $\bA_{0}^{(1)} = \bA_{0}$. For $\bB_0^{(1)}$, we have
\begin{equation*}
    \begin{split}
        \bB_{0}^{(1)} & = \bB_{0} + \frac{h}{2} \left(- \bP_{\bB_0}^{\text{null}} \nabla_{\bW}\mL_0\bA_0^{\top}\left(\bA_0 \bA_0^{\top}\right)^{-1} - \bB_0\bX_0\right)\\
        & = \bB_{0} + \frac{h}{2} \left(- \bP_{\bB_0}^{\text{null}} \left(\bW_0\bs - \by\right)\bs^{\top}\bA_0^{\top}\left(\bA_0 \bA_0^{\top}\right)^{-1} - \bB_0\bX_0\right).
    \end{split}
\end{equation*}
Note that $\left\| \bW_0\bs - \by\right\| = \Theta(1)$, $\left\| \bA_0 \bs\right\| = \Theta(1)$, and $\left\|\left(\bA_0 \bA_0^{\top}\right)^{-1} \right\| = \Theta(1)$, we have 
\begin{equation*}
    \left\|\bP_{\bB_0}^{\text{null}} \left(\bW_0\bs - \by\right)\bs^{\top}\bA_0^{\top}\left(\bA_0 \bA_0^{\top}\right)^{-1}\right\|=\Theta(1).
\end{equation*}
Here, we obtain that $\left\|\bB_{0}^{(1)}(:,j) \right\|=\Theta(1)$ and thus $\left\|\bB_{0}^{(1)\top}\bB_{0}^{(1)} \right\|=\Theta(1)$, since $\bP_{\bB_0}^{\text{null}}=\bI$. 
We has already shown that  $(\bA_0^{(1)}, \bB_0^{(1)})$ satisfies the desired scaling property in \eqref{eq29}. 
Applying the remaining RK4 updates and repeating the same scaling analysis as in the general induction step, we conclude that conditions in \eqref{eq29} are satisfied at $t=1$. 
Now, both $(\bA,\bB)$ are in the same scale in the sense of $\bA\bA^{\top}$ and $\bB^{\top}\bB$. If project them to the manifold $\mathcal{M}$, all these properties still preserve. 
Therefore, the base case holds. Combined with the induction argument established earlier, this proves that ODELoRA-RK4 achieves stable feature learning for all $t \in \mathbb{N}_{+}$.

\end{document}